\documentclass[11pt]{article}
\usepackage[utf8]{inputenc}
\usepackage{fullpage}
\usepackage{amsmath,amsthm,amssymb}
\usepackage{enumerate}
\usepackage{xcolor}
\usepackage{cite}
\usepackage{url}
\usepackage{xspace}
\usepackage{hyperref}
\hypersetup{
    colorlinks=true
}
\usepackage{commands}
\usepackage[utf8]{inputenc}
\usepackage{amsmath, amssymb, bbm, graphicx, url}
\usepackage{amsfonts}
\usepackage{float}
\usepackage{graphicx}
\usepackage{bm}
\usepackage{amssymb}
\usepackage{bbm}
\usepackage{xcolor}
\usepackage{algorithm}
\usepackage{algorithmic}
\usepackage{dsfont}
\usepackage{xspace}
\usepackage{comment}
\usepackage{tikz}

\newcommand{\eat}[1]{}

\newtheorem{result}{Theorem} \newtheorem{minorresult}[result]{Proposition}

\newtheorem*{definition*}{Definition}
\newtheorem*{proposition*}{Proposition}
\newtheorem*{corollary*}{Corollary}

\newtheorem{theorem}{Theorem}[section]
\newtheorem{lemma}[theorem]{Lemma}
\newtheorem{definition}[theorem]{Definition}

\newtheorem{proposition}[theorem]{Proposition}
\newtheorem{corollary}[theorem]{Corollary}

\newcommand{\Thr}{\mathrm{Thr}}
\newcommand{\Int}{\mathrm{Int}}

\newcommand{\wh}{\widehat}

\newcommand{\X}{\mathcal{X}}
\newcommand{\Y}{\mathcal{Y}}
\newcommand{\mC}{\mathcal{C}}
\newcommand{\mD}{\mathcal{D}}
\newcommand{\mH}{\mathcal{H}}

\newcommand{\mT}{\mathcal{T}}

\newcommand{\mF}{\mathcal{F}}
\newcommand{\mG}{\mathcal{G}}

\newcommand{\mL}{\mathcal{L}}
\newcommand{\mW}{\mathcal{W}}

\newcommand{\mU}{\mathcal{U}}

\newcommand{\tf}{\tilde{p}}
\newcommand{\tp}{\tilde{p}}
\newcommand{\ps}{\ensuremath{p^*}}

\newcommand{\y}{\mathbf{y}}
\newcommand{\ty}{\mathbf{\tilde{y}}}
\newcommand{\sy}{\mathbf{y^*}}

\newcommand{\Lin}{\mathrm{Lin}}
\newcommand{\mb}[1]{\mathbf{#1}}

\newcommand{\x}{\mathbf{x}}
\newcommand{\z}{\mathbf{z}}

\newcommand{\rgta}{\rightarrow}

\newcommand{\lt}{\left}
\newcommand{\rt}{\right}

\newcommand{\zo}{\ensuremath{\{0,1\}}}
\newcommand{\izo}{\ensuremath{[0,1]}}

\newcommand{\infnorm}[1]{\left\lVert#1\right\rVert_\infty}

\renewcommand{\tilde}{\wt}
\renewcommand{\hat}{\wh}
\renewcommand{\eps}{\varepsilon}

\newcommand{\la}{\ensuremath{\langle}}
\newcommand{\ra}{\ensuremath{\rangle}}

\newcommand{\pmo}{\ensuremath{ \{\pm 1\} }}
\newcommand{\fr}[1]{\ensuremath{\frac{1}{#1}}}

\newcommand{\err}{\mathrm{err}}

\newcommand{\C}{\mathcal{C}}

\DeclareMathOperator{\poly}{poly}

\DeclareMathOperator*{\E}{\mathbf{E}}

\DeclareMathOperator*{\argmin}{arg\,min}

\newcommand{\Lip}{\mathrm{Lip}}
\newcommand{\level}{\mathrm{level}}

\newcommand{\mpk}[1]{}
\newcommand{\PG}[1]{}

\newcommand{\sign}{\mathrm{sign}}

\newcommand{\bE}{{\mathbb{E}}}

\newcommand{\card}[1] {\left\vert #1 \right\vert}
\newcommand{\set}[1] {\left\{ #1 \right\}}

\newcommand{\Ber}{\mathrm{Ber}}

\newcommand{\MAcc}{\mathrm{MA}}
\newcommand{\MCal}{\mathrm{MC}}
\newcommand{\MAp}{\mathrm{calMA}}
\newcommand{\WL}{\mathrm{WL}}
\newcommand{\calg}{\mathrm{reCAL}}
\newcommand{\Lglm}{\mathcal{L}_{\mathrm{GLM}}}
\newcommand{\mLa}{\mL_{\mathrm{all}}}

\newcommand{\pone}{p^\delta}
\newcommand{\ptwo}{\bar{p}}
\newcommand{\pth}{\hat{p}}

\newcommand{\ECE}{\mathsf{ECE}}
\newcommand{\estECE}{\mathsf{estECE}}
\newcommand{\malg}{\mathsf{MA}}
\newcommand{\MAerr}{\mathsf{MAE}}
\newcommand{\CalErr}{\mathsf{CE}}

\title{Loss Minimization through the Lens of\\ Outcome Indistinguishability\footnote{Appears in Innovations in Theoretical Computer Science (ITCS) 2023.}}
\author{
\qquad Parikshit Gopalan \\
\qquad Apple
\and
Lunjia Hu\thanks{\textbf{LH} is supported by the Simons Foundation Collaboration on the Theory of Algorithmic Fairness, Omer Reingold’s NSF Award IIS-1908774, and Moses Charikar's Simons Investigator award.}\\
Stanford University
\and
Michael P. Kim\thanks{\textbf{MPK} is supported by the Miller Institute for Basic Research in Science and, in part, by the Simons Collaboration on the Theory of Algorithmic Fairness.} \qquad\\
UC Berkeley \qquad
\and
Omer Reingold\thanks{\textbf{OR} is supported by the Simons Foundation Collaboration on the Theory of Algorithmic Fairness, the Simons Foundation investigators award 689988, and Sloan Foundation Grant 2020-13941.} \\
Stanford University
\and
Udi Wieder\\
VMware}

\date{}

\begin{document}

\maketitle

\begin{abstract}

We present a new perspective on loss minimization and the recent notion of Omniprediction through the lens of Outcome Indistingusihability.
For a collection of losses and hypothesis class, omniprediction requires that a predictor provide a loss-minimization guarantee simultaneously for every loss in the collection compared to the best (loss-specific) hypothesis in the class.
We present a generic template to learn predictors satisfying a guarantee we call \emph{Loss Outcome Indistinguishability}.
For a set of statistical tests---based on a collection of losses and hypothesis class---a predictor is Loss OI if it is indistinguishable (according to the tests) from Nature's true probabilities over outcomes.
By design, Loss OI implies omniprediction in a direct and intuitive manner.
We give a decomposition of Loss OI into two modular conditions: the first is implied by calibration, the second is equivalent to multiaccuracy for a class of functions derived from the loss and the hypothesis class.
By careful analysis of this class, we give efficient constructions of omnipredictors for many interesting classes of loss functions. 

This decomposition highlights the utility of a new multi-group fairness notion that we call calibrated multiaccuracy, which lies in between multiaccuracy and multicalibration.
We show that calibrated multiaccuracy implies Loss OI for the important set of convex losses arising from Generalized Linear Models, without requiring full multicalibration.
For such losses, we show an equivalence between our computational notion of Loss OI and a geometric notion of indistinguishability, formulated as \emph{Pythagorean theorems} in the associated Bregman divergence.
We give an efficient algorithm for calibrated multiaccuracy with computational complexity comparable to that of multiaccuracy.
In all, calibrated multiaccuracy offers an interesting tradeoff point between efficiency and generality in the omniprediction landscape.

 \end{abstract}
\thispagestyle{empty}
\newpage
\setcounter{page}{1}

\section{Introduction}
\label{sec:intro}

\renewcommand{\C}{\mathcal{C}}
\renewcommand{\L}{\mathcal{L}}
\newcommand{\D}{\mathcal{D}}
\newcommand{\pt}{\tilde{p}}
\newcommand{\ys}{\y^*}
\newcommand{\yt}{\tilde{\y}}
\newcommand{\Hcal}{\mathcal{H}}

Loss minimization is the dominant paradigm in machine learning. Techniques for loss minimization have played a critical role in the development of the theory and practice of supervised learning \cite{kearns1994introduction, boydBook, shalev2012online, shalev2014understanding, hardt2021patterns}.
A clean theoretical formulation of the underlying problem is via the notion of agnostic PAC learning \cite{shalev2014understanding}. We consider real-valued loss functions $\ell$ that take two arguments, a label $y \in \zo$ and an action $t \in \R$. Given a loss $\ell$, a base class of hypotheses $\C$, and approximation parameter $\eps$, the goal is to find a hypothesis $h$ that achieves near-optimal expected loss (compared to $c \in \C$) over a fixed, but unknown distribution $\D$:\footnote{This version where we do not restrict $h$ to belong to $\mC$ is sometimes called improper learning.  }
\begin{gather*}
\E_{(\x,\ys) \sim \D}[\ell(\ys,h(\x))] \le \min_{c \in \C}\E_{(\x,\ys) \sim \D}[\ell(\ys,c(\x))] + \eps.
\end{gather*}
Researchers have devoted significant effort into developing different choices of loss functions \cite{masnadi2008design}.
Different settings---so the conventional wisdom goes---require the design of different loss functions (e.g., squared, zero-one, logistic) to better encode the objectives of the task at hand (regression, classification, calibration). The choice of loss function dictates  the updates during training and hence the resulting loss minimizer. With different loss functions, there are many different optimal hypotheses, and one needs to learn afresh for each loss.

\eat{
Agnostic learning is the central problem of supervised learning theory.
Typically, agnostic learning is formulated through loss minimization.
Given a loss $\ell$, which measures the quality of predictions, a base class of hypotheses $\C$, and approximation parameter $\eps$, the goal is to find a hypothesis $h$ that achieves near-optimal expected loss (compared to $c \in \C$) over a fixed, but unknown distribution $\D$.
\begin{gather*}
\E_{(\x,\ys) \sim \D}[\ell(\ys,h(\x))] \le \min_{c \in \C}\E_{(\x,\ys) \sim \D}[\ell(\ys,c(\x))] + \eps
\end{gather*}
Since its formulation \cite{valiant1984theory,haussler1992decision,kearns1994toward}, developing techniques for loss minimization has played a critical role in the development of the theory and practice of supervised learning \cite{kearns1994introduction, boydBook, shalev2012online, shalev2014understanding, hardt2021patterns}.

Within the study of loss minimization, researchers have devoted significant effort into developing different choices of loss functions \cite{masnadi2008design}.
Different settings---so the conventional wisdom goes---require the design of different loss functions (e.g., squared, zero-one, logistic) to better encode the objectives of the task at hand (regression, classification, calibration).
The choice of loss function dictates the resulting loss minimizer, so different loss functions typically mean there are many different optimal hypotheses, and one needs to learn afresh for each loss.
}
Recent work pushes back against this conventional wisdom.
The work of \cite{omni} introduces a solution concept for agnostic PAC learning, which they call \emph{omniprediction}.
Intuitively, an omnipredictor $\pt:\X \to [0,1]$ is a predictor that can be used to simultaneously minimize loss for many different losses.
Formally, an omnipredictor is parameterized by a collection of loss functions $\L$, a class of hypotheses $\C$, and approximation parameter $\eps$.
Given any loss $\ell \in \L$, a decision-maker can treat $\pt(x)$ as if it were the Bayes optimal predictor $p^*(x) = \E[\y|x]$, selecting an action $t$ that will minimize $\E [\ell(\ty,t)]$ where $\ty$ is drawn according to $\tp$. Even though the true labels are drawn according to $p^*(x)$, the resulting decision rule is $\eps$-optimal for $\ell$ over $c \in \C$.
Importantly, the omnipredictor $\pt$ is a single prediction function, fixed in advance, but yields optimal decisions for all $\ell \in \L$. The Bayes optimal predictor $p^*(x)$ is easily seen to be an omnipredictor for all losses, the question is whether they can be learnt efficiently. The main result in \cite{omni} is a sweeping feasibility result:  they demonstrate that for any efficiently learnable hypothesis class $\C$ and $\eps > 0$, efficient
omnipredictors exist for the class $\L_\mathrm{cvx}$ of all Lipschitz, convex loss functions.
They prove this by showing a connection to \emph{multicalibration}, from the literature on fair prediction \cite{hkrr2018}.

Multicalibration was developed with the goal of promoting fairness across subpopulations encoded by a class of functions $\C$. In contrast to the loss-minimization paradigm, multicalibration does not frame learning as loss minimization. Rather, the goal of learning is to satisfy a collection of ``indistinguishability'' constraints.
This view on multicalibration was developed in the recent work of \cite{oi}, who introduced an alternative paradigm for learning called \emph{outcome indistinguishability} (OI).
OI  considers two {\em alternate worlds} on individual-outcome pairs:  in the natural world, outcomes $(\x,\ys)$ are generated by Nature's true joint distribution; in the other simulated world, outcomes $(\x,\yt)$ are sampled according to the predictive model $\yt \sim \Ber(\pt(\x))$.
OI requires the learner to produce a predictor $\pt$ in which the two worlds are computationally indistinguishable.
\newcommand{\A}{\mathcal{A}}
More formally, OI is parameterized by a class of distinguisher algorithms $\A$.
Each $a \in \A$ receives an individual $x \in \X$, an outcome $y \in \set{0,1}$, and the prediction $\pt(x)$ and outputs a value in the interval $[0,1]$.
For such a collection of algorithms $\A$ and approximation parameter $\eps$, a predictor $\pt$ is $(\A,\eps)$-outcome indistinguishable\footnote{In fact, \cite{oi} introduce a more general hierarchy of OI notions, whose levels are based on the distinguishers' access to the predictions given by $\pt$. The variant where we allow distinguishers access to $\pt(\x)$ (so-called, \emph{sample-access} OI) is known to be computationally \emph{equivalent} to multicalibration.} if no algorithm $a \in \A$ can distinguish between the two distributions over individual-outcome pairs.
\begin{gather*}
    \E_{(\x,\ys) \sim \D}[a(\x,\ys,\pt(\x))] \approx_\eps \E_{\substack{\x \sim \D\\\yt \sim \Ber(\pt(\x))}}[a(\x,\yt,\pt(\x))]
\end{gather*}

As multicalibration is a special case of OI, by the results of \cite{omni}, one can view omniprediction for convex, Lipschitz losses as a consequence of OI, for an appropriate family of distinguishers.
While rigorous, this argument is rather indirect and in our view, it does not provide clear intuition for why there should be a link between loss minimization and indistinguishability. 
Moreover, the connection to multicalibration  established in \cite{omni} is rather constrained in terms of the family of loss functions $\mL$.
If we want omnipredictors for a more expressive class such as all Lipschitz functions, not just convex ones (where it is known that multicalibration is insufficient \cite[Lemma 6.7]{omni}), or simpler omnipredictors for a more restricted class of convex loss functions (such as $L_p$ losses), the results of prior work don't shed much light on how we might proceed.

\subsection{Our Contributions}
\label{sec:cont}

Motivated by ominprediction, we establish a direct and intuitive connection between loss minimization and outcome indistinguishability, through a notion which we call \emph{Loss OI}.
Fundamental to our approach is to use loss functions as tools to construct \emph{distinguishers}:
given a family $\mL$ of loss functions and a family of hypotheses $\mC$, we devise a family of distinguishers $\mU_{\mL, \mC} = \{u_{\ell,c}\}_{\ell \in \mL, c \in \mC}$ such that if $\tp(x)$ is not an omnipredictor, then some distinguisher from this family can tell apart the labels generated by Nature from those generated by the predictor's simulation.
We say that any predictor that fools every distinguisher from this family satisfies loss OI. By construction, loss OI implies omniprediction.

We show that loss OI admits a decomposition into two simpler outcome indistinguishability requirements which we call \emph{hypothesis OI} and \emph{decision OI}. Hypothesis OI compares the expected loss of the hypothesis $c$ when labels are generated by Nature versus its simulation by $\tp$, for each hypothesis in the class $c \in \C$.
Decision OI tests compares the expected loss incurred when we take actions based on the optimal post-processing of the predictions of $\tp$ under the two distributions on labels.
We give a characterization of these indistinguishability conditions in terms of the \emph{discrete derivative} $\partial \ell:[0,1] \to \R$ of the loss function $\ell$, defined as $\partial \ell(t) = \ell(1,t) - \ell(0,t)$. Via this characterization, decision OI amounts to a {\em weighted} calibration condition derived from $\partial \ell$, which is implied by standard notions of calibration. Hypothesis OI can be expressed as a \emph{multiaccuracy} condition for the class of functions $\partial \mL \circ \mC = \set{\partial \ell \circ c : \ell \in \L, c \in \C}$.
Multiaccuracy \cite{hkrr2018,kgz} for a given hypothesis family $\mC$ is a weaker notion than multicalibration for $\mC$. Both notions require access to a weak agnostic learner for $\mC$, but multiaccuracy admits simpler and more efficient algorithms in terms of sample complexity and running time.

\paragraph{Loss OI for specific families.}

With this decomposition, we turn our attention to specific collections of loss functions $\L$. Since decision OI follows from calibration, to achieve hypothesis OI and loss OI, we analyze the structure of $\partial \L \circ \mC$, with the goal of bounding the complexity of such functions.

\begin{itemize}
    \item \textbf{All losses:} We begin with the family  $\L_\mathrm{all}$ of \emph{all} losses  satisfying minimal boundedness conditions. The losses need not be convex or Lipschitz.
    We show that loss OI is possible for $\L_\mathrm{all}$ and any hypothesis class $\mC$,  provided we can ensure calibration and multiaccuracy over functions on the level sets of $\C$.
    Specifically, we require multiaccuracy over the collection $\mathrm{level}(\C) = \set{f \circ c}$ for all $c \in \C$ and all maps $f:[-1,1] \to [-1,1]$. We can view these as the set of all bounded functions over the level sets of $c$.
    This has immediate consequences for Boolean (even discrete) hypothesis classes, since there, the class $\mathrm{level}(\C)$ is not much more complex than $\mC$ itself:  $\C$-multiaccuracy plus calibration implies loss minimization for any loss function.

    \item {\bf Lipschitz losses: }
    Under Lipschitzness (but still without convexity), a  weaker multiaccuracy condition suffices.
We define $\Int(\C,\alpha)$ to be the collection of Boolean functions, which are the indicators of the events that $c(x)$ lies in an interval of width $\alpha$. 
    We show that for Lipschitz losses, $\partial \L \circ \mC$ lies in the linear span of functions in $\Int(\mC,\alpha)$.
    Hence, calibration together with $\Int(\C,\alpha)$-multiaccuracy guarantees loss OI for all Lipschitz loss functions.
    
    \item {\bf GLM losses:}
GLMs are a popular class of convex loss minimization based models, which include basic learning algorithms such as linear and logistic regression. They can be viewed as minimizing Bregman divergences for predictors which are derived from linear combination of $\mC$.  For the class of GLM losses $\L_{\mathrm{GLM}}$, we show that $\partial \L \circ \mC = \mC$. Hence, calibrated multiaccuracy---that is, calibration together with $\mC$-mulitaccuracy---guarantees loss OI for all GLM losses.
    We give an equivalence between of predictors that satisfy Loss OI for $\L_{\mathrm{GLM}}$ and the set of predictors satisfying a certain Pythagorean Theorem in the geometry of the corresponding Bregman divergence. 
    
    Finally, we exhibit a reverse connection by showing that the optimal solution to any $L_1$-regularized GLM loss minimization problem is multiaccurate. This leads us to fast and practical methods for achieving both multiaccuracy and calibrated multiaccuracy. 
\end{itemize}

Our results for Loss OI are incomparable with the result of \cite{omni} on omnipredictors. On one hand, loss OI is stronger than omniprediction. On the other hand, we require weak agnostic learning for $\partial \L \circ \mC$, which might be a much more powerful primitive than weak learning for $\mC$ itself (which is sufficient for multicalibration). For the class of convex Lipschitz losses $\L_{\mathrm{cvx}}$ considered in \cite{omni}, we show that multicalibration does not imply loss OI, although it implies omniprediction. Our best ``upper bound'' for  $(\L_\mathrm{cvx},\C)$-loss OI comes from $\Int(\C,\alpha)$-multiaccuracy, and it applies even when the losses are non-convex. For the subset $\L_{\mathrm{GLM}}\subset \L_{\mathrm{cvx}}$, we  show a stronger guarantee (loss OI versus omniprediction) from weaker assumptions (calibrated multiaccuracy versus multicalibration).

\eat{
Surprisingly, full multicalibration is not necessary to obtain nontrivial omniprediction guarantees, which can follow from $\C$-multiaccuracy plus calibration alone.
Multiaccuracy plus calibration has been hinted at in prior works \cite{dwork2019rankings,GargKR19,KimThesis}, living somewhere between multiaccuracy and multicalibration, but has not been studied in detail.
Our results on omniprediction demonstrate its importance as a solution concept in its own right.
We conclude the manuscript by establishing the algorithmic feasibility of multiaccuracy + calibration, with complexity is closer to that of multiaccuracy than full multicalibration.
}

\paragraph{Calibrated multiaccuracy.}

A key takeaway from our results is the surprising power of the notion of calibrated multiaccuracy, where we require predictors to satisfy both multiaccuracy with respect to $\mC$ and calibration. It implies loss OI for the class of GLM losses, and for the case when $\mC$ is Boolean. As a group fairness notion, it lies in between the notions of multiaccuracy and multicalibration. We show the running time and sample complexity needed to achieve calibrated multiaccuracy are not much higher than  that required for multiaccuracy, by giving a simple algorithm that alternates between ensuring multiaccuracy is achieved (using gradient descent for squared loss), and recalibrating the output. The key insight is that either of these steps reduces the squared loss of the predictor. Hence the number of invocations of the weak learner is not much more in the worst case from that required to achieve multiaccuracy, and significantly smaller than that required for multicalibration. 

\paragraph{Perspective.}

We see the key contribution of our work as conceptual: we bring the OI lens to the problem of loss minimization. Reasoning about the simulated labels $\ty$ turns out to a powerful idea in this context, which has not been explored before, even in prior work on omniprediction. Our framework leverages this to give a {\em compiler} that translates loss OI for a pair $(\mL, \mC)$ into {\em low-level} calibration and multiaccuracy conditions. With this setup, the proofs of our results are not technically hard. For instance, our result for GLMs uses the well-known fact that the loss function for any GLM has the form $\ell_g(y,t) = g(t) - yt$. It follows that $\partial \ell(t) = -t$, hence $\mC$-multiaccuracy suffices for hypothesis OI (assuming $\mC$ is closed under negation).

The loss OI perspective establishes a natural and versatile link between loss minimization and indistinguishability. It broadens our understanding of omniprediction. On one hand, it shows it can be scaled up beyond convex, Lipschitz losses. But it can also scaled down for more limited classes of loss functions to give more efficient constructions.  It enables a range of omniprediction guarantees, where the richness of the collection of losses scales with the expressive power of the class for which we require multiaccuracy.

\paragraph{Structure of this manuscript: }
The remainder of the manuscript is structured as follows. In Section \ref{sec:overview}, we present a high-level technical overview of our definitions and results. We discuss related work in \ref{sec:related}. 
In Section~\ref{sec:prelims}, we give preliminaries and formal background.
In Section~\ref{sec:oi-loss}, we introduce Loss OI and its relationship to omniprediction and the other notions of indistinguishability. We then show how Loss OI can be formulated in terms of multiaccuracy and calibration.
In Section~\ref{sec:glms}, we instantiate present our main result on loss OI for Generalized linear models. We also show an equivalence between our formulation of Loss OI for GLMs and Pythagorean theorems in the geometry of Bregman divergences. In Section \ref{sec:gen-loss}, we consider other families of loss functions including those that are not necessarily convex or Lipschitz.
In Section~\ref{sec:algo}, we present and analyze an efficient algorithm for calibrated multiaccuracy, and establish that it is more efficient than multicalibration.  We report on the results from some preliminary experiments that aim to establish the efficiency and effectiveness of calibrated multiaccuracy in Section \ref{sec:exp}.  Proofs are occasionally deferred to Appendix \ref{sec:app1} to streamline the flow.

\section{Technical Overview}
\label{sec:overview}

In this section, we give a more detailed but still high-level explanation of how loss OI gives a indistinguishability viewpoint on loss minimization and omniprediction.
The starting point for our investigation is understanding why the Bayes optimal predictor is an omnipredictor for any loss and concept class.
We use $\ps:\X \to [0,1]$ to denote the Bayes optimal predictor, which represents Nature's true probability of positive outcomes.
\begin{gather*}
    \ps(x) = \E_{(\x,\ys) \sim \D}[\ys \vert \x = x]
\end{gather*}
We consider loss functions $\ell:\set{0,1} \times [0,1] \to \R^+$ that take a label and action as arguments and return a real valued loss. For such a loss $\ell$, if the labels are drawn as $\y \sim \Ber(p)$, there exists an optimal action $k_\ell(p) \in [0,1]$ defined as
\begin{gather*}
    k_\ell(p) = \argmin_{t \in [0,1]} \E_{\y \sim \Ber(p)}[\ell(\y,t)]
\end{gather*}
We refer to $k_\ell$ as the optimal post-processing for $\ell$. Since the Bayes optimal predictor $\ps$ governs the conditional distribution over outcomes $\y^*$, by averaging over $\x \sim \D$, we conclude that $k_\ell \circ \ps$ satisfies the loss minimization guarantee for any loss, with respect to any hypothesis class $\C$.
\begin{gather}
\label{eq:bayes}
    \E_{(\x,\ys) \sim \D}[\ell(\ys, k_\ell(\ps(\x)))] \le \min_{c \in \mC} \E_{(\x,\ys) \sim \D}[\ell(\ys, c(\x))]
\end{gather}

The challenge of constructing an omnipredictor is, given specific families of losses $\mL$ and hypotheses $\mC$ respectively, to identify properties of $\tp$ that will allow us to replace $\ps$ with $\tp$ in the above statement, as long as $\ell \in \mL$ and $c \in \mC$. Formally, we say that a predictor $\pt$ is an $(\L,\mC,\eps)$-omnipredictor if for every loss $\ell \in \L$, the post-processed predictor $k_\ell \circ \pt$ is an $\eps$-loss minimizer compared to the class $\mC$:
\begin{gather}
\label{eq:omni}
    \E_{(\x, \ys) \sim \D}[\ell(\ys, k_\ell(\pt(\x)))] \le \min_{c \in \mC} \E_{(\x,\ys) \sim \D}[\ell(\ys, c(\x))] + \eps.
\end{gather}

\subsection{Omniprediction from  outcome indistinguishability.}

Omniprediction is a statement about Nature's distribution.
Equation \eqref{eq:omni} makes no mention of the simulated predictions $\ty$.
It is unclear how considering labels $\ty$ from the predictor's simulation might be useful. Indeed, the simulated labels do not play a role in the \cite{omni} derivation of omniprediction from multicalibration. 

The key insight is that {\em in the simulated world of labels $\ty$, $\tp$ is the Bayes optimal predictor}. So Equation \eqref{eq:omni} holds with $\eps =0$.  Indeed, we just apply Equation \eqref{eq:bayes} with $\ys =\ty$ and $\ps =\tp$ to get
\begin{gather}
\label{eq:bayes2}
    \E_{\substack{\x\sim \D\\ \ty \sim \Ber(\tp(\x))}}[\ell(\ty,k_\ell(\tp(\x)))] \le \min_{c \in \C} \E_{\substack{\x\sim \D\\ \ty \sim \Ber(\tp(\x))}}[\ell(\ty,c(\x))]
\end{gather}
If $\tp$ has the property that the expectations on either side of the Equation don't change much when we replace $\ty$ with $\ys$, then this will imply our desired omniprediction guarantee (Equation \eqref{eq:omni}). But this condition is a form of outcome indistinguishability, tailored to distinguishers constructed from $\mL$ and $\mC$. Loss OI is a crisp formulation of this notion.

\eat{
Omniprediction can be viewed as some sort of indistinguishability condition: an omnipredictor is a predictor $\pt$ that ``fools'' the proof that $k_\ell \circ \ps$ is a loss minimizer.
What remains, then, is to identify the properties of $\ps$ that are used in the proof of optimality---focusing attention to losses from $\L$ and hypotheses from $\C$---and to translate these properties into a set of ``tests'' that suffice to ensure the proof still goes through for $\pt$.
Optimistically, this approach will help us to identify a more precise (i.e.,\ efficient) set of constraints than enforcing full multicalibration.
}

\paragraph{Loss OI.}
\newcommand{\U}{\mathcal{U}}
Loss OI is parameterized by a loss class $\L$ and a concept class $\C$, which induce the following collection of distinguishers:
\begin{align}
\label{eq:def-u}
u_{\ell,c}(y,p,x) &= \ell(y,c(x)) - \ell(y,k_\ell(p))\\
\U_{\L,\C} &= \set{u_{\ell,c} : \ell \in \L, c \in \C}\notag
\end{align}
For a given loss $\ell$, the distinguisher $u_{\ell,c}:\Y \times [0,1] \times \X \to \R$ measures the excess loss of the prediction $c(x)$ compared to the optimal post-processing $k_\ell$ applied to the predicted label distribution $p$.
For a fixed $x \in \X$, if we generated labels $\yt \sim \Ber(\pt(x)$, then $k_\ell(\pt(x))$ is the optimal action, so $u_{\ell,c}(\yt,\pt(x),x) \ge 0$. Hence, the expected value over $\x \sim \mD$ is also non-negative. For omniprediction to hold, it would suffice if 
\begin{align*}
    \E_{(\x, \y^*) \sim \mD} [u_{\ell, c}(\ys, \pt(\x), \x)] = \E_{(\x, \y^*) \sim \mD}[\ell(\ys, c(\x))] - \E_{(\x, \y^*) \sim \mD}[\ell(\y^*, k_\ell(\tp(\x)))]\geq 0.
\end{align*}
Loss OI imposes the stronger condition that the expectation under Nature's distribution and the simulation are (approximately) equal. For a loss class $\L$, a concept class $\C$, $\eps > 0$, a predictor $\pt$ is $(\L,\C,\eps)$-loss OI
if for all $\ell \in \L$ and for all $c \in \C$, the following approximate equality holds.
\begin{gather}
    \label{eqn:lossOI}
    \E_{\substack{\x \sim \D\\\yt \sim \Ber(\pt(\x))}}[u_{\ell,c}(\yt,\pt(\x),\x)] \approx_\eps \E_{(\x,\ys) \sim \D}[u_{\ell,c}(\ys,\pt(\x),\x)]
\end{gather}
By design, Loss OI guarantees omniprediction. In fact, it is a strictly stronger notion. In Section~\ref{sec:separation}, we show that while $\mC$-multicalibration implies omniprediction for $\L_{\mathrm{cvx}}$,  it does not imply loss OI even for the $\ell_4$ loss.

\begin{minorresult}
\label{result:LOI2OP}
If a predictor $\pt$ is $(\L,\C,\eps)$-loss OI, then $\pt$ is an $(\L,\C,\eps)$-omnipredictor. The converse does not hold.
\end{minorresult}

\eat{
The indistinguishability framework allows us to ensure that if the decision rule is optimal on modeled outcomes, it will also be optimal on Nature's outcomes.
That is, we hope to enforce the following condition for each potential $u_{\ell,c} \in \U_{\L,\C}$.
\begin{gather*}
\E_{\substack{\x \sim \D\\\yt \sim \Ber(\pt(\x))}}[u_{\ell,c}(\yt,\pt(\x),\x)] \ge 0 \implies \E_{(\x,\ys) \sim \D}[u_{\ell,c}(\ys,\pt(\x),\x)] \ge 0.
\end{gather*}
One way to ensure that the sign of the expectation of the distinguishers is the same is to ensure that the expectations are equal.
This observation leads us to the definition of Loss OI.}

The \cite{omni} proof of omniprediction was tailored specifically to multicalibration and the specific class of convex loss functions $\L_\mathrm{cvx}$. In contrast, Loss OI is a versatile notion that may be applied to any class of loss functions.
By approaching the question of omniprediction via loss OI, we arrive at an easy-to-state set of sufficient conditions to obtain omniprediction for any class of losses $\L$ and hypothesis class $\C$.

\paragraph{Characterizing Loss OI via calibration and multiaccuracy.}

We define loss OI using distinguisher functions $\set{u_{\ell,c}}$ that depend on both $c(\x)$ and $\pt(\x)$.
It is known from the work of \cite{oi} that when distinguishers receive simultaneous access to $c(\x)$ and $\pt(\x)$, outcome indistinguishability can implement (full) multicalibration. However, the distinguishers $u_{\ell,c}$ have very specific structure, which permits a decomposition of loss OI into two modular conditions, involving two different distinguishers that each depend on the label and one out of $c(\x)$ and $\pt(\x)$ \emph{separately}.
The first set of distinguishers will simply compare the loss of hypotheses $c \in \C$ for each loss $\ell \in \L$, a condition we call \emph{hypothesis OI}.
\begin{gather}
\label{eqn:modelOI}
    \E[\ell(\yt, c(\x))] \approx_\eps \E[\ell(\ys,c(\x))]
\end{gather}
The second set of distinguishers evaluates the loss achieved by the predictor $\pt$ under optimal post-processing for each loss, a condition we call \emph{decision OI}.
\begin{gather}
\label{eqn:decisionOI}
    \E[\ell(\yt,k_\ell(\pt(\x)))] \approx_\eps \E[\ell(\ys,k_\ell(\pt(\x)))]
\end{gather}
Subtracting (\ref{eqn:decisionOI}) from (\ref{eqn:modelOI}), we obtain (\ref{eqn:lossOI}), albeit with a slightly larger error parameter.
In other words, if $\pt$ satisfies both hypothesis OI and decision OI, then $\pt$ satisfies loss OI.

It turns out that decision OI is easy to achieve, we show that it is implied by calibration. Recall that a predictor is $\alpha$-calibrated if $\E[\y|\tp(\x) =v] \approx_\alpha v$. 
Using a more nuanced notion called weighted calibration from \cite{GopalanKSZ22}, we can get an exact characterization of decision OI (see Theorem \ref{thm:dec+model}). 

To present a characterization of hypothesis OI, we need a couple of definitions. 
For a class of functions $\mC$ and approximation $\alpha \ge 0$, a predictor $\pt$ is $(\mC,\alpha)$-multiaccurate if for every $c \in \mC$, the correlation between $c$ and $\ys - \pt(\x)$ is at most $\alpha$. Formally, we require
\begin{gather*}
    \card{\E[c(\x) \cdot (\ys - \pt(\x))]} \le \alpha.
\end{gather*}
For a loss function $\ell$, we define the discrete derivative $\partial \ell$ as $\partial{\ell}(t) = \ell(1, t) - \ell(0,t)$. For a loss class $\L$ and hypothesis class $\C$, we consider the class of functions $\partial\L\circ \C = \set{\partial \ell \circ c : \ell \in \L, c \in \C}$. We can characterize Hypothesis OI in terms of $\partial\L \circ \C$-multiaccuracy.
\begin{minorresult}
(Decomposition for Loss OI)
For loss class $\L$, hypothesis class $\C$, and $\eps \ge 0$, predictor $\pt$ is $(\mL, \mC, \eps)$-hypothesis OI iff it is $(\partial\L\circ \C, \eps)$-multiaccurate. Thus, if $\pt$ is $\eps$-calibrated and $(\partial\L\circ \C, \eps)$-multiaccurate, then it is $(\L, \C, O(\eps))$-loss OI, and hence an $(\L, \C, O(\eps))$-omnipredictor.
\end{minorresult}

Thus we have decomposed loss OI into two constraints on our predictors: calibration, and multiaccuracy for the class $\partial \L \circ \mC$. This presents an  alternative (and possibly more efficient) route to obtaining omnipredictors than via multicalibration.

\paragraph{Non-convex losses.}

Using our decomposition theorem we show that, perhaps surprisingly, loss-OI and omniprediction are feasible even for non-convex losses, given a sufficiently powerful learner for functions derived from $\mC$. We require the losses to be bounded: $\infnorm{\partial \ell} \leq 1$. But otherwise, the losses can be arbitrary, we do not assume Lipschitzness or convexity. Define the set
\begin{gather*}
    \level(\C) = \set{f \circ c : f \in \mathcal{F},\ c \in \C}\ \ 
    \textrm{where } \mathcal{F} = \set{f:\Im(\mC) \to [-1,1]}
\end{gather*}
That is, $\level(\C)$ consists of all possible bounded post-processings of $c \in \C$; in particular the functions $f \in \mathcal{F}$ only get to distinguish between the \emph{level sets} of each $c \in \C$. The importance of $\level(\C)$ stems from the fact that $\partial \ell \circ c$ belongs to this class,
hence $\level(\C)$-multiaccuracy suffices for Hypothesis-OI over all loss functions.
\begin{minorresult}
For any class of loss functions $\L$, if $\pt$ is $(\level(\C),\alpha)$-multiaccurate, then $\pt$ is $(\L,\C,\alpha)$-hypothesis OI. Hence if $\pt$ is $\alpha$-calibrated and $(\level(\C),\alpha)$-multiaccurate, then for any loss class $\L$, $\pt$ is $(\L,\C,O(\alpha))$-loss OI.
\end{minorresult}

Thus, omnipredictors for every bounded loss function are computable, with complexity scaling with the complexity of weak agnostic learning for $\level(\C)$. 
While $\level(\C)$ could in general be far more expressive than $\C$ itself, there are important special cases, including when $\mC$ is a family of Boolean functions, where it is not much larger than $\mC$. In these settings, we get loss-OI for arbitrary losses from calibration and $\mC$-multiaccuracy. This includes natural loss functions such as weighted $0$-$1$ loss which are important for classification.

\eat{ Consider the multiaccuracy condition for the class of functions $\partial\L(\C)$, without any assumptions on $\ell \in \L$.
\begin{gather*}
    \E[\partial\ell(c(x)) \cdot (\ys - \pt(\x))]
    =
    \E[(\ell(1,c(x)) - \ell(0,c(x)) \cdot (\ys - \pt(\x))] \approx 0
\end{gather*}
Clearly, if there is a choice of $c(x) = t$, such that $\partial\ell(t)$ is unbounded, then there will be no way to audit this condition from samples.
But short of assuming the discrete derivative is bounded, we can hope to reason about these test functions.}

\paragraph{Lipschitz losses}
If we are willing to assume that the losses are Lipschitz, then we can obtain hypothesis OI from a weaker multiaccuracy condition.
Intuitively, if the loss $\ell$ is Lipschitz in $t$, then so is $\partial \ell$, so we only need to consider Lipschitz post-processings.
We can achieve this guarantee by enforcing multiaccuracy over the class of functions $\Int(\C,\alpha)$ which are the indicators of the event that $c(x)$ lies in a certain interval $I \subset [-1,1]$ of width $\alpha$, over all $c \in \mC$ and intervals $I$.

\eat{
Specifically, consider partitioning the range $[-1,1]$ into $m$ discrete intervals of width $\delta = 2/m$, $\mathcal{I}^\delta = \set{I_j^\delta : j \in \set{1,\hdots, m}}$.
We define $\Int(\C,\delta)$ as follows.
\begin{gather*}
    \Int(\C,\delta) = \set{\ind{c(\cdot) \in I_j^\delta} : c \in \C,\ I_j^\delta \in \mathcal{I}^\delta}
\end{gather*}}

We show that $\Int(\C,\alpha)$-multiaccuracy suffices to give Hypothesis OI for Lipschitz losses.
\begin{minorresult}
For any class of $1$-Lipschitz loss functions $\L$, if $\pt$ is $(\Int(\C,\alpha),\alpha^2)$-multiaccurate then $\pt$ is $(\L,\C,O(\alpha))$-hypothesis OI. If $\pt$ is also calibrated, then $\pt$ is $(\L,\C,O(\alpha))$-loss OI.
\end{minorresult}

\subsection{Loss OI in GLMs.}
GLMs are a important class of models from statistics that generalize linear and logistic regression \cite{MN89, GLMbook}. On a technical level, GLMs are constructed using the following recipe:
\begin{enumerate}
    \item We start with an arbitrary monotone increasing transfer function $g':\R \to \R$ whose range contains $[0,1]$. 
    \item Note that the integral $g(t)$ of $g'(t)$ is convex since $g'$ is monotone. We define its matching loss $\ell_g(y,t) = g(t) -yt$ which is a convex function of $t$ \cite{AuerHW95}.
    \item We look for the model $h \in \mH$ which minimizes $\E[\ell_g(y, h(x))]$ where the hypothesis class $\mH$ is taken to be linear combinations over some base class $\mC$.  This gives rise to a convex optimization problem that can be solved efficiently \cite{GLMbook, GLMnotes}.
\end{enumerate}

When we take $g'(t) = t$, this recipe gives linear regression with the squared loss. When $g'(t) = \sigma(t)$ is the sigmoid, we get logistic regression. The class of losses $\L_{\mathrm{GLM}}$ that arise in this manner are convex. Thus, by the results of \cite{omni}, $\C$-multicalibration suffices to obtain omniprediction for $\L_{\mathrm{GLM}}$.

Our first result on GLMs shows that the class $\partial\L_{\mathrm{GLM}} \circ \C = \C$. This holds because every loss $\ell_g \in \Lglm$ has the form  $\ell_g(y,t) = g(t) - yt$, hence $\partial \ell(t) = -t$ is linear in $t$. This  means that $\C$-multiaccuracy---not a derived class---plus calibration suffices for loss OI for GLMs.
\begin{result}[Informal]
If $\pt$ is $(\C,\alpha)$-multiaccurate and calibrated, then it is $(\L_\mathrm{GLM},\C,O(\alpha))$-Loss OI.
\end{result}
These results highlight the power of  calibrated multiaccuracy which gives omniprediction for all GLM losses. Before this, we only knew how to achieve this using the stronger notion of multicalibration. Is it really much easier to achieve calibrated multiaccuracy? A key piece of the answer comes from our next result shows a reverse connection between multiaccuracy and GLM optimality with $\ell_1$-regularization.  We state the result informally here.
\begin{minorresult}[Informal]
For any GLM loss and $\alpha > 0$, the optimizer of the $\ell_1$-regularized GLM optimization over the class $\mC$ is $(\C,\alpha)$-multiaccurate.
\end{minorresult}

This result immediately gives a (number of) efficient avenues for computing a $\C$-multiaccurate predictor: run any $\ell_1$-regularized GLM learner, like regularized logsitic regression or Lasso \cite{tibshirani1996regression}. It also suggests a template for achieving calibrated multiaccuracy:
we can alternate between the GLM learner and a recalibration procedure  until convergence. We will analyze a simple algorithm based on this template and show that its complexity is comparable to that of achieving multiaccuracy, and considerably lower than what is needed to achieve multicalibration.

Finally, we consider the Loss OI conditions for GLM losses. We show that, in this setting, the computational indistinguishability notion of Loss OI is equivalent to a geometric indistinguishability condition, formalized by Pythagorean theorems in the associated Bregman divergence. We state the result informally below, deferring the definitions of technical terms to the later sections.

\begin{result}[Informal]
Let  $g'$ be strictly monotonically increasing, let $f$ be the Legendre dual of $g$, and let $D_f$ be the corresponding Bregman divergence.
A predictor $\pt$ is $(\ell_g,\Hcal,\alpha)$-Loss OI if and only if the following approximate Pythagorean theorem holds approximately.
\begin{gather*}
    \E[D_f(\ps(\x),g'(h(\x)))] \approx_\alpha \E[D_f(\ps(\x),\pt(\x))] + \E[D_f(\pt(\x),g'(h(\x)))]
\end{gather*}
\end{result}
Intuitively, the Pythagorean theorem says that the ``distance'' between $\ps$ and a predictor derived from the class $\Hcal$ can be broken down into ``orthogonal'' components:  the distance between $\ps$ and $\pt$ plus the distance between $\pt$ and the predictor from $\Hcal$.
In other words, if a predictor $\pt$ is $\C$-multiaccurate and calibrated, then it is simultaneously a ``projection'' of the best GLMs towards the statistically optimal predictor $\ps$.

\subsection{Algorithms for Calibrated Multiaccuracy.}

For a given hypothesis clas $\mC$, we define the following classes of predictors.
\begin{itemize}
    \item Let $\MAcc(\alpha)$ denote the set of predictors that are $(\mC, \alpha)$-multiaccurate.
    \item Let $\MAp(\alpha)$ denote the set of predictors that are $\alpha$-calibrated and $(\mC, \alpha)$-multiaccurate. 
    \item Let $\MCal(\alpha)$ denote the set of predictors that are $(\mC, \alpha)$-multicalibrated. 
\end{itemize}
Then we have $\MAcc(\alpha) \supseteq \MAp(\alpha) \supseteq \MCal(\alpha)$. We compare the complexity of computing a predictor in each of these classes given access to a $(\rho, \sigma)$-weak learner for $\mC$ \cite{SBD2, kalai2005boosting, KalaiMV08}. Such a learner, when given access to a distribution $(\x, \z)$ where $\x \sim \mD_\X$ and $\z$ are labels in $\pmo$, if there exists $c \in \mC$ such that $\E[c(\x)\z] \geq\rho$, will return $c'$ such that $\E[c'(\x)\z] \geq \sigma$. If no such $c$ exists it returns $\perp$. The complexity of learning the predictor in any of the aforementioned classes is governed by the number of oracle calls to the weak learner. 

We present Algorithm \ref{alg:map} for achieving calibrated multiaccuracy that alternates between ensuring multiaccuracy (using the weak learner), and calibrating the predictor. The key insight that makes it efficient is that either step can be seen to reduce the same potential function, which is the squared distance from the Bayes optimal predictor. This results in a worst-case complexity for $\MAp$ that is not too different than just for achieving the weaker guarantee of $\MA$ (since that algorithm is also analyzed using the same potential). 

We compare the number of oracle calls needed for computing a predictor in each of $\MA$, $\MAp$ and $\MC$. We emphasize that this is a comparison between the best known upper bounds. For $\MA$, we use the \cite{hkrr2018} algorithm as analyzed in Lemma \ref{lem:hkrr}. For $\MAp$, we use our analysis of Algorithm \ref{alg:map} in Theorem \ref{thm:dec+model}. For $\MC$, we use the analysis of the algorithm from \cite[Section 9]{omni}, which is derived from the boosting by branching programs algorithm by \cite{MansourM2002}.

\begin{itemize}
    \item For $\MA(\alpha)$, the number of calls made by the algorithm of \cite{hkrr2018} is bounded by $O(1/\sigma^2)$. 
    \item For $\MAp(\alpha)$, the number of calls made by Algorithm \ref{alg:map} bounded by $O(1/\sigma^2)$. 
    \item For $\MC(\alpha)$, the number of calls made by the algorithm of \cite{omni} is bounded by $O(1/\alpha^2\sigma^4)$. The weak learning assumption required is also somewhat stronger, see Section \ref{sec:algo} and Appendix \ref{app:omni} for a detailed discussion. 
\end{itemize}

The comparison above shows that $\MA$ and $\MAp$ have similar complexities in terms of the worst-case number of calls to the weak learner. The number of calls required for $\MC$ is significantly larger.  These results suggest that calibrated multiaccuracy is an interesting multi-group notion in its own right, that lies in between $\MA$ and $\MC$. It offers an interesting tradeoff point between efficiency and generality in the omniprediction landscape. It is an interesting open problem to ask if it captures any of the desirable fairness properties of $\MC$, or even of low-degree multicalibration \cite{GopalanKSZ22}. 

Finally, we show that calibrated multiaccuracy (and hence omniprediction for GLM losses) cannot be achieved by any algorithm that outputs a hypothesis which is a Single Index Model (SIM): these are functions of the form $u(\sum w_c c(x))$ where $u$ is monotonically increasing. In particular, this implies that known algorithms like the Isotron \cite{KalaiS09, KakadeKKS11} which work in the realizable setting but produce a SIM as hypothesis cannot give an omnipredictor in the non-realizable setting.

We present some preliminary experiments which support the efficiency and omniprediction claims  in Section \ref{sec:exp}. Importantly, the implementation is fewer than 100 lines of python code using standard regression and calibration libraries in sklearn, whereas multicalibration is more complex \cite{gopalan2021multicalibrated}. For a collection of common losses (inclduing some non-GLM losses),  the calibrated MA predictor always competes with and sometimes outdoes the best linear predictor tailored to the loss.

\subsection{Related Work and Discussion}
\label{sec:related}

Our work is inspired by and most closely related to the work of \cite{omni} which introduced omnipredictors, and the outcome indistinguishability framework of \cite{oi}. The relation of our results to the former is detailed in depth in Section \ref{sec:cont}. The outcome indistinguishability framework establishes general connections between multi-group fairness notions and appropriate levels in OI hierarchy. Here, we use their framework to focus on more fine-grained notions of OI that are tailored towards loss minimization and omniprediction. The framework of Loss OI is quite versatile, and has already been extended by \cite{kim2022making} to the ``performative'' prediction setting, where predictions can influence the distribution over outcomes.

Rothblum and Yona~\cite{RothblumYona21} employed the notion of outcome indistinguishability in order to obtain loss-minimization over a rich family of sub populations. Their notion of loss functions is more general than ours. But they fix a single loss function in their discussion whereas we seek to address general families of loss functions. A major distinction is that our work studies the complexity of loss OI for broad families of loss functions and relates them to distinguishers that do not depend on the loss function.

The work of \cite{GopalanKSZ22} on low-degree multicalibration was also motivated by the goal of finding intermediate notions of multigroup fairness between $\MA$ and $\MC$. They propose the hierarchy $\{\MC_d\}$ of degree-$d$ multicalibrated predictors which interpolates between these two notions. They show that several desirable fairness properties of $\MC$ are already achieved at low levels of the hierarchy, at a computational cost similar to that of $\MA$.  Our results on calibrated multiaccuracy are similar in spirit but incomparable, we show how omniprediction for some important convex losses can already be obtained at $\MAp$,  at a computational cost comparable to that of $\MA$.

There is a vast body of work on Generalized Linear Models \cite{GLMbook, GLMnotes}. Classically, the focus is on the setting where the transfer function $g'$ or equivalently its inverse $f'$ known as the {\em link function} are known. To every such transfer function, one can associate a convex macthing loss $\ell_g$ \cite{AuerHW95}.  The resulting program can be solved using the iteratively reweighted least squares algorithm \cite{GLMnotes, MN89}. We will denote the  set of convex matching losses arising in this manner by $\Lglm$ s. The more challenging setting is where the link function is unknown. This is sometimes called the SIM (single index model) problem in the literature. To our knowledge, all work with provable guarantees (prior to the work of \cite{omni}) hold only for the realizable setting: the data are generated so that $\E[\y^*|\x] = g'(h(\x))$ for some $h \in \Lin(\mC)$, both $g'$ and $h$ are unknown.  The first provable guarantees in this scenario were given by Kalai \cite{Kalai04}, who finds a hypothesis that is close in squared error to the ground truth $g'\circ h$, and is represented as  branching program. The elegant Isotron algorithm for this problem was introduced and analyzed  in \cite{KalaiS09, KakadeKKS11}, it is a proper learning algorithm where the output is of the form $u \circ \tilde{h}$, where $\tilde{h} \in \Lin(\mC)$ and $u$ is monotone.

Both our work and the work of \cite{omni} depart from these works in that they do not require the realizability assumption. We give a single predictor $\tf$, with the guarantee that for any transfer function $g'$ (satisfying certain technical conditions), the matching loss $\ell_g$ of the post-processing predictor is comparable to that incurred by the best $h \in \Lin(\mC)$. Under the realizability assumption, for any Lipschitz transfer function $g'$, bounding the matching loss implies a squared loss bound \cite{VKnotes}. In the agnostic setting, squared loss and bounds on the matching loss are incomparable. The works of \cite{SSS11, GoelKKT17} apply polynomial kernel techniques to the problem of squared loss minimization when the transfer function is sigmoid or the ReLU for families of losses including $\ell_1$ and the squared loss. In these settings, a polynomial dependence on the accuracy parameter $\eps$ is not possible.

Bregman divergences and Pythagorean theorems for them are studied in information geometry \cite{Nielsen, coverT}, although the term is broadly used for inequalities arising from projections onto convex bodies. That a stronger guarantee than omniprediction holds true for the squared loss was observed in the work of \cite[Lemma 8.4]{omni}. This guarantee was subsequently shown to hold even with degree-$2$ multicalibration \cite[Proposition A.1]{GopalanKSZ22}. Our results generalize this to all GLM losses, and only assumes calibrated multiaccuracy, while also showing that for such losses, Pythagorean theorems are equivalent to loss OI.

\mpk{Write about concurrent work.}

\section{Preliminaries}
\label{sec:prelims}

Let $\mD$ be a distribution on labelled examples $(\x, \y^*)$ comprising of points $\x$ from a domain $\X$ and binary outcomes\footnote{All our results can be extended to multi-class setting where there are finitely many distinct classes, but we work with the binary setting for simplicity.} $\y^* \in \zo$.
We let $\mD_\X$ denote the marginal distribution over $\X$.
We will occasionally refer to the distribution $\mD$ as Nature.
We assume sample access to Nature. $\Ber(p)$ denotes the Bernoulli distribution on $\{0,1\}$ with parameter $p$. For a real valued function $f: \mT \to \R$, let $\infnorm{f} = \max_{\mT}|f(x)|$. For a family of such functions $\mF$, let $\infnorm{\mF} = \max_{f \in \mF}\infnorm{f}$.

\paragraph{Predictors:}
A predictor is a function $\tf:\X \rgta [0,1]$ be a predictor, where $\tf(x)$ is interpreted as an estimate of the label being $1$, conditioned on $x$. For a predictor $\pt:\X \to [0,1]$, 
we define the distribution $(\x,\yt) \sim \mD(\tp)$ on $\X \times \zo$ where $\x \sim \mD_\X$ is sampled according to Nature's marginal distribution over inputs and conditioned on $\x$, $\yt \sim \Ber(\pt(\x))$ so that
\begin{gather*}
\pt(x) = \E[\yt \vert \x = x].
\end{gather*}
We use $\ps(x) \in [0,1]$ to denote the Bayes optimal prediction for an individual $x \in \X$.
\begin{gather*}
    \ps(x) = \E[\ys \vert \x = x]
\end{gather*}
In other words, using the optimal predictor $\D(\ps) = \D$ recovers the true distribution, Nature.

\paragraph{Calibration:}
Intuitively, a predictor is calibrated if, conditioned on the prediction $\pt(\x) = v$, the expected outcome is close to $v$.
\begin{gather*}
    \E[\ys \vert \pt(\x) = v] \approx v
\end{gather*}
Formally, we quantify approximate calibration through \emph{expected calibration error}.
\begin{definition} (ECE and Approximate calibration)
    We define the expected calibration error (ECE) of a predictor $\tp$ as
    \[ \ECE(\tp) = \E_{\tp(\x)} \lt| \E_{\y |\tp(\x)}[\y - \tp(\x)] \rt|. \]
    For $\alpha \geq 0$, a predictor $\tp: \X \rgta [0,1]$ is $\alpha$-calibrated if $\ECE(\tp) \leq \alpha$.
\end{definition}
A predictor $\tp$ is perfectly calibrated if $\alpha =0$, so that $\E_\mD[\y^*|\tp(\x) =v] = v$.
While the notion of approximate calibration is well-defined for all predictors, checking for calibration efficiently requires the predictor to be discretized.
When efficiency is a consideration, we will assume that the supported values of the predictor are multiples of some $\delta \in [0,1]$; such assumptions are standard in the calibration literature \cite{FosterV98, hkrr2018}.
For such predictors, one can check for $\alpha$-calibration given black-box access to $\tp$ in time $\poly(1/\alpha, 1/\delta)$, using labeled samples.

Following \cite{GopalanKSZ22}, we will allow for weighted notions of calibration, 
parametrized by a family of weight functions $\mW = \{w:[0,1] \rgta \R\}$. 
Intuitively, we think of a weight function as highlighting predictions belonging to certain regions of $[0,1]$.
\begin{definition}
    Let $\mW = \{w:[0,1] \rgta \R\}$ be a family of weight functions. For a predictor $\tp:\X \rgta [0,1]$ we define
    \[\CalErr(\mW, \tp) = \max_{w \in \mW} \lt| \E_\mD[ w(\tp(\x))(\sy - \tp(\x))] \rt|. \]
\end{definition}

We collect some simple properties of  weighted calibration in the next lemma, the proof is in Section \ref{app:prelims}. The first is that
$\ECE$ is captured by considering weight functions bounded in absolute value by $1$. 
The second is that $\alpha$-calibration implies a bound on $\CalErr(\mW, \tp)$ for any bounded family of weights $\mW$.  

\begin{lemma}
\label{lem:fully}
\begin{enumerate}
\item Let $\mW^f$ denote the space of all functions $w: [0,1] \rgta [-1,1]$.
Then
\[ \ECE(\tp) = \CalErr(\mW^f, \alpha).\] 
\item If $\tp$ is $\alpha$-calibrated, then for any family $\mW$ of weight functions, 
\[ \CalErr(\mW, \tp) \leq   \infnorm{\mW}\alpha. \]
\end{enumerate}
\end{lemma}

We will sometimes use weaker notions of calibration. An important special case is where we take $\mW_1$ to be the set of all $1$-Lipschitz weight functions bounded in the range $[-1,1]$. We say that a predictor $\tp$ is $\alpha$-smoothly calibrated if it $\CalErr(\mW_1, \tp) \leq \alpha$.

\eat{
\mpk{$d$ may be confused for the dimension of $\X$.  Let's use something else.  $m$?  $b$?  Is multiclass just for generality, or is there a significance in any results for it?}
We use $\bE_d = \{e_1, \ldots, e_d\}$ to denote the standard basis for $\R^d$, which can be thought of as one-hot encodings of the labels. We denote $(\x, \y^*) \sim \mD^*$ to denote a draw from this distribution, where $\x \in \X$ and $\y^* \in \bE_d$ denotes the one hot encoding of the label. 
Let $\Delta_d$ denote the space of probability distributions on $d$ labels. 
}

\paragraph{Loss functions and decision functions:} 
A loss function is a function $\ell: \zo \times \R \rgta \R$. For instance, we define the squared loss by $\ell_2(y, t) = \|y -t\|_2^2$ and the $\ell_p$ loss by $\ell_p(y, t) = \|y - t\|_p^p$. We define $b_\ell$, the Lipschitz constant of $\ell$, to be the smallest constant so that $\abs{\ell(y, t_1) - \ell(y, t_2)} \leq b_\ell |t_1 - t_2|$.
We let $\Lip_{b}$ denote the set of all $b$-Lipschitz functions. We say that a loss $\ell$ is convex, if for each $y \in \zo$, $\ell(y,t)$ is a convex function of $t$. In a generic loss minimization problem, given a loss function $\ell$ and a class $\mH$ of hypotheses, one tries to find the hypothesis  $h \in \mH$ which minimizes $\E[\ell(\y, h(\x))]$. We extend the definition of $\ell$ via linearity so that the first argument can take values in $[0,1]$. We define
\[ \ell(p,t) = \E_{\y \sim \Ber(p)}[\ell(\y, t)] = p\cdot \ell(1,t) + (1 - p)\cdot \ell(0,t). \]

 A decision function is a function $k:[0,1] \rgta \R$. We think of $k$ as taking predictions $p \in [0,1]$ from a predictor and mapping them to actions $k(p) \in \R$. Decision functions are used to select a suitable action for a loss function, given a prediction of the distribution of labels. For a loss $\ell$, we define the Bayes-optimal decision function $k_\ell: [0,1] \rgta \R$ by
\[ k_\ell(p) = \argmin_{t \in \R}\ell(p,t). \]
For proper losses like the squared error $(y - t)^2$, $k_\ell$ is simply the identity function.
For the $\ell_1$ loss $\card{y - t}$, $k_{\ell_1}(p)$ rounds $p$ to the nearest value in $\set{0,1}$.

\paragraph{Hypotheses:} A bounded hypothesis class is a family of functions $\mC \subseteq \{ c: \X \rgta [-1,1]\}$. We will assume that $\mC$ contains the constant function $1$ and is closed under negation. 
Our results will typically assume some learnability properties of the class $\mC$, such as having bounded dimension and being weakly learnable.
We define the class $\Lin(\mC, B)$ to contain all functions of the form
\[ h(x) = \sum_{c \in \mC} w_c c(x), \ \ \sum_{c \in \mC}|w_c| \leq B. \]
Note that $|h(x)| \leq B$ for all $h \in \Lin(\mC, B)$.  We will consider loss minimization problems with the hypothesis class $\mH = \Lin(\mC, B)$ (e.g linear or logistic regression). Here $B$ can be viewed as a regularization parameter.

\paragraph{Multicalibration:}
Originally introduced as a form of ``multi-group'' fairness \cite{hkrr2018}, \emph{multicalibration} and related notions have seen application beyond fair prediction in recent years.
Intuitively, multicalibration requires that the predictions of $\pt$ appear calibrated even when we restrict our attention to structured subpopulations.
\cite{hkrr2018} formalizes the collection of subpopulations through a concept class $\C$. Importantly, the multicalibration guarantee holds simultaneously for every $c \in \C$.

First, we define a weaker notion called multiaccuracy \cite{hkrr2018, kgz}, which requires that predictions appear accurate in expectation (unbiased) over each $c \in \C$.
\begin{definition}
    Let $\mC = \{c:\X \rgta [-1,1]\}$ be a family of hypotheses and $\alpha \geq 0$. We say that the predictor $\tf:\X \rgta [0,1]$ is $(\mC, \alpha)$-multiaccurate if for every $c \in \mC$ it holds that
    \[\lt| \E_{\mD}[ c(\x)(\sy - \tf(\x))] \rt| \leq \alpha \]
\end{definition}

Multicalibration strengthens both calibration and multiaccuracy, requiring approximate calibration over each $c \in \C$.
We adapt the definitions in \cite{hkrr2018, omni} to our notion of approximate calibration.
\begin{definition}
    Let $\mC = \{c:\X \rgta [-1,1]\}$ be a family of hypotheses and $\alpha \geq 0$. We say that the predictor $\tf:\X \rgta [0,1]$ is $(\mC, \alpha)$-multicalibrated if for every $c \in \mC$ it holds that
    \[\E_{\tp(\x)}\lt| \E_{\y|\tp(\x)}[c(\x)(\sy - \tf(\x))] \rt| \leq \alpha \]
\end{definition}

By averaging over the predicted values, we can see that $(\C,\alpha)$-multicalibration implies $(\mC, \alpha)$-multiaccuracy. Since we assume $1 \in \C$, $(\C,\alpha)$-multicalibration also implies $\alpha$-calibration. 

In defining multiaccuracy and multicalibration, we assume that the hypotheses are bounded by $1$ in absolute value. For general hypotheses families $\mH$, we define the multiaccuracy error as
\[\MAerr(\mH, \tp) = \max_{h \in \mH}\lt[ \lt| \E_{\mD}[h(\x)(\y^* - \tp(\x))]\rt| \rt]. \]
We will generally reserve the term $(\mC, \alpha)$-multiaccuracy to denote a bounded hypothesis class $\mC$ where $\MAerr(\mC, \tp) \leq \alpha$. The hypotheses classes $\mH$ most relevant to us are of the form  $\mH = \Lin(\mC, B)$. For these, we can derive bounds on the multiaccuracy error from bounds for the base hypotheses in $\mC$, that decay linearly with $B$. The proof is via linearity of expectation.

\begin{lemma}
\label{lem:linear}
    If the predictor $\tf$ is $(\mC, \alpha)$-multiaccurate, then for $B \geq 1$ and $\mH = \Lin(\mC, B)$ we have \[ \MAerr(\mH, \tp) \leq B\alpha.\]
\end{lemma}

\paragraph{Omnipredictors:} 

The notion of omniprediction introduced by \cite{omni} asks for a single predictor which can do as well as the best hypothesis  in a hypothesis class $\mH$ for a family $\mL$ of loss functions. 

\begin{definition}
    We say that the predictor $\tf:\X \rgta [0,1]$ is an $(\mL, \mH, \delta)$-omnipredictor if for every loss $\ell \in \mL$ and hypothesis $h \in \mH$, 
    \[ \E[\ell(\y^*, k_\ell(\tf(\x)))] \leq \E[\ell(\y^*, h(\x))] + \delta.\] 
\end{definition}

\paragraph{Outcome Indistinguishability:}

Outcome indistinguishability  introduced by \cite{oi} provides an elegant framework for reasoning about the quality predictions made by a predictor $\tf$, by measuring their ability to fool statistical tests when nature's labels $\sy$ and replaced by simulated labels $\ty$.
The notion is parameterized by a class of algorithms $\mathcal{A} \subseteq \set{a:\X \times \zo \times [0,1] \to [-1,1]}$, whose goal is to ``distinguish'' Nature's distribution and the modeled distribution.
\begin{definition}[Outcome Indistinguishability]
\label{def:oi}
    A predictor $\pt:\X \to [0,1]$ is $(\mathcal{A},\eps)$-outcome indistinguishable if for every $a \in \mathcal{A}$,
    \begin{gather*}
        \card{\E_{(\x,\ys) \sim \D}[a(\x,\ys,\pt(\x))] - \E_{(\x,\yt) \sim \D(\pt)}[a(\x,\yt,\pt(\x))]} \le \eps.
    \end{gather*}
\end{definition}
In fact, \cite{oi} consider various levels of OI which are defined by the degree of access to the predictions made available to the tests.
In their language, Definition~\ref{def:oi} corresponds to ``sample-access OI'' where the distinguisher receives access to $\x$, $\pt(\x)$, and outcomes sampled either from $\ys \sim \Ber(\ps(\x))$ or $\yt \sim \Ber(\pt(\x))$.

Also of relevance to us are special cases of this model.
The first, so-called ``no-access OI'' corresponds to a restriction where the distinguishers do not receive $\pt(\x)$, and simply has access to either $(\x^*, \y^*) \sim \mD$ or $(\x, \ty) \sim \mD(\tf)$.
Sample-access OI and No-access OI are in tight correspondence with multicalibration and multiaccuracy, respectively \cite{oi}.
Another interesting special case of sample-access OI is when we are given access to $\tp(\x)$ but not to the point $\x$.
Here, the goal is to distinguish between $(\y^*, \tf(\x)) \sim \mD$ and $\ty, \tf(\x)) \sim D(\tf)$.
OI for this model is tightly connected to calibration: for boolean outcomes, it follows that perfect calibration implies that these distributions are identical. \section{Outcome Indistinguishability for loss functions}
\label{sec:oi-loss}

We define notions of outcome indistinguishability for a predictor $\tf$  with regard to distinguishers that are derived from a loss function $\ell$. We allow distinguishers that take on real values, such a function distinguishes two distributions if its expected values differ significantly between them.

We define the notion of Loss OI formally.
Here we compare the difference (between Nature and the predictor's model) in the expected loss suffered when using the hypothesis $c \in \mC$ compared to when using the Bayes-optimal decision function $k_\ell$ based on the predictor $\tf$.
\begin{definition}
    {\bf (Loss OI)} 
    Let $\mL$ be a family of loss functions, $\C$ be a family of hypotheses, and $\eps > 0$.
    For each $\ell \in \mL, c \in \C$, define the distinguisher $u_{\ell,c}: \zo \times [0,1] \times \X \rgta \R$ by
    \begin{align} 
        \label{eq:loss-dist}
        u_{\ell,c}(y, \tp(x), x) = \ell(y, c(x)) - \ell(y, k_\ell(\tp(x)). 
    \end{align}
    We say that the predictor $\tf$ is $(\mL, \C, \eps)$-loss-OI if for every loss $\ell \in \mL$ and hypothesis$c \in \C$, 
    \[ \abs{\E_\mD[u_{\ell, c}(\y^*, \tf(\x), \x)] - \E_{\mD(\tp)}[u_{\ell, c}(\ty, \tf(\x), \x)]} \leq \eps.\] 
\end{definition}

We define two additional, simpler notions.
First is that of decision OI, which informally states that applying the Bayes optimal  decision functions to the predictions of $\tf$ and computing the expected loss cannot distinguish between $\y^*$ and $\ty$.

\begin{definition}
    {\bf (Decision OI)} 
    Let $\mL$ be a family of loss functions, and $\eps > 0$. 
    We say that predictor $\tf$ is $(\mL, \eps)$-decision-OI if for every $\ell \in \mL$ it holds that
    \[ \abs{\E_\mD[\ell(\y^*, k_\ell(\tf(\x))] - \E_{\mD(\tp)}[\ell(\ty, k_\ell(\tf(\x))]} \leq \eps. \]
\end{definition}

Our next notion is hypothesis OI, which stipulates that no hypothesis from $\mC$ results in significantly different expected loss whether the labels come from nature or  the simulation. 

\begin{definition}
    {\bf (Hypothesis OI)} 
    Let $\mL$ be a family of loss functions, $\mC$ a family of hypotheses  and $\eps > 0$.
    We say that the predictor $\tf$ is $(\mL, \mC, \eps)$-hypothesis-OI for $\eps \geq 0$ if for loss $\ell \in  \mL$ and every hypothesis $c \in \mC$ it holds that
    \[ \abs{\E[\ell(\y^*,c(\x))] - \E[\ell(\ty, c(\x))]} \leq \eps. \]
\end{definition}

We show that Loss OI is implied by having both Decision OI and Hypothesis OI simultaneously.
\begin{lemma}
\label{lem:implies}
{(\bf Decomposition lemma)}
    If the predictor $\tf:\X \rgta [0,1]$ is $(\mL, \eps_1)$-decision-OI and $(\mL, \mC, \eps_2)$-hypothesis-OI, then it is  
    $(\mL, \mC, \eps_1 + \eps_2)$-loss-OI.
\end{lemma}
\begin{proof}
For each $\ell \in \mL$ and $c \in \C$ we can write
\begin{align} 
\label{eq:rewrite}
&\E[u_{\ell,c}(\y^*, \tf(\x), \x)] - \E[u_{\ell,c}(\ty, \tf(\x), \x)]\notag\\
&=\E[\ell(\y^*,c(\x)) - \ell(\y^*, k_\ell(\tf(\x)))] - \E[\ell(\ty,c(\x)) - \ell(\ty, k_\ell(\tf(\x)))]\notag\\
&= \E[(\ell(\y^*,c(\x))] -\E[\ell(\ty,c(\x))] + \E[\ell(\ty, k(\tf(\x)))] - \E[\ell(\y^*, k_\ell(\tf(\x)))].
\end{align}
Hence by the triangle inequality, 
\begin{align*}
    \abs{\E[u_{\ell,c}(\y^*, \tf(\x), \x)] - \E[u_{\ell,c}(\ty, \tf(\x), \x)]} &\leq \abs{\E[(\ell(\y^*,c(\x))] -\E[\ell(\ty,c(\x))]} \\
    & + \abs{\E[\ell(\ty, k(\tf(\x))))] - \E[\ell(\y^*, k_\ell(\tf(\x)))] } \\
    &\leq \eps_1 + \eps_2.
\end{align*}
where the first term is bounded by hypothesis-OI and the second is bounded by decision-OI.
\end{proof}

\subsection{Loss-OI implies Omniprediction}
\label{sec:separation}

Our interest in the notion of loss-OI stems from the fact that it implies omniprediction.

\begin{proposition}[Formal Restatement of Proposition~\ref{result:LOI2OP}]
    \label{prop:omni}
    If the predictor $\tf:\X \rgta [0,1]$ is $(\mL, \mC, \eps)$-loss-OI, then it is an $(\mL, \mC, \eps)$-omnipredictor.  
\end{proposition}
\begin{proof}
      A consequence of loss-OI is that for every $\ell \in \mL$ and $c \in \mC$, we have
      \begin{align}
      \label{eq:ineq1}
          \E[u_{\ell,c}(\y^*, \tf(\x), \x)] \geq  \E[u_{\ell,c}(\ty, \tf(\x), \x)] - \eps.
      \end{align}
      But for every $x \in \X$, by the definition of the Bayes-optimal decision function $k_\ell$ we have
      \begin{align*}
          \E[u_{\ell,c}(\ty, \tf(\x), \x)|\x = x] = \E[\ell(\ty, c(\x)) - \ell(\ty, k_\ell(\pt(\x))|\x = x] \geq 0
      \end{align*}
      since $k_\ell(\tf(x))$ is defined to be action that minimizes expected loss for $\ty \sim \Ber(\tf(x))$. Averaging over all $\x \sim \mD$ gives 
      \begin{align*}
          \E[u_{\ell,c}(\ty, \tf(\x), \x)] = \E[\ell(\ty, c(\x))] - \E[\ell(\ty, k_\ell(\pt(\x))] \geq 0.
      \end{align*}
      Plugging this into Equation \eqref{eq:ineq1} gives
    \begin{align*}
          \E[u_{\ell,c}(\y^*, \tf(\x), \x)] = \E[\ell(\y^*, c(\x)) - \ell(\y^*, k_\ell(\pt(\x))] \geq -\eps.
    \end{align*}
    Rearranging, we get that for every $\ell \in \mL, c \in \mC$,  
    \begin{align*}
        \E[\ell(\y^*, k_\ell(\pt(\x)))] \leq \E[\ell(\y^*, c(\x))] + \eps. 
    \end{align*}
    hence $\tf$ is an $(\mL, \mC, \eps)$-omnipredictor.
\end{proof}

The converse of this statement is not true. We show that omniprediction does not imply Loss-OI for any class $\mL$ than includes the $\ell_4$ loss. We prove an even stronger statement, that multicalibration does not imply loss-OI. This statement is stronger because of the result of \cite{omni} that multicalibration implies omniprediction for a broad class of convex loss functions. \eat{While their result holds for all convex losses with some mild Lipschitzness property, we will only consider the $\ell_p$ losses for simplicity. 
}
We define the $\ell_p$ loss for all $p \geq 1$ as
\[ \ell_p(y,z) = \fr{p}|y - z|^p \]
where the normalization by $p$ makes it $1$-Lipschitz. Let  $L_p = \{\ell_p\}_{p \ge 1}$. 
We prove the following result which separates multicalibration from loss OI.
\begin{theorem}
\label{thm:sep-mc-oi}
    There exist a distribution $\mD$, a class $\mC$ and a predictor $\tf$ such that 
    \begin{itemize}
        \item $\tf$  is $(\mC, 0)$-multicalibrated, so it is an $(L_p, \mC, 0)$-omnipredictor. 
        \item $\tf$ is not $(\{\ell_4\}, \mC, \eps)$-loss OI for any $\eps < 4/9$.  
    \end{itemize}
\end{theorem}

The proof which is given in Section \ref{app:oi-loss} uses Fourier analysis on the Boolean cube. 

\subsection{Loss OI from Calibration and Multiaccuracy}
\label{sec:oi-ma+}

In order to analyze the notions of OI, we need to compare the expected loss under different distributions on labels for a certain action. The notion of {\em discrete derivative} of a loss function will aid these comparisons.

\begin{definition}
\label{def:discrete-derivative}
Given a loss $\ell: \zo \times \R \rgta \R$,  define the function $\partial \ell: \R \rgta \R$ as
\begin{align*} 
    \partial{\ell}(t) &= \ell(1, t) - \ell(0,t).
\end{align*}
\end{definition}

The following lemma justifies the analogy to partial derivatives. 
\begin{lemma}
\label{lem:disc-taylor}
For random variables $\y, \y' \in \zo$, and $t \in \R$ we have
\begin{align}
    \label{eq:disc-taylor}
    \E[\ell(\y, t)] - \E[\ell(\y',t)]  = \E[(\y - \y')\partial \ell(t)].
\end{align} 
\end{lemma}
\begin{proof}
    By definition
    \[\E[\ell(\y,t)] =  \E[\y \ell(1, t) + (1 - \y)\ell(0,t)] = \E[\y \partial \ell(t)] + \ell(0, t) \] 
    We write a similar expression for $\y'$ and subtract. 
\end{proof}

We now present characterizations of decision-OI  and hypothesis-OI in terms of weighted calibration and multiaccuracy errors for suitably defined classes of functions. Combined with Lemma \ref{lem:implies}, this gives a decomposition of loss OI as a calibration condition and a multiaccuracy condition. 

\begin{theorem}
\label{thm:dec+model}
Let $\mL$ be a family of loss functions and $\mC$ be a hypothesis class.
\begin{enumerate}
    \item  Define the family of hypotheses  $\partial \mL \circ \mC = \{ \partial \ell \circ c\}_{\ell \in \mL, c \in \mC}$. The predictor $\tf:\X \rgta [0,1]$ is $(\mL, \mC, \eps_1)$-hypothesis-OI where $\eps_1 =\MAerr(\mC', \tp)$. 
    \item Define the family of weight functions $\mW' = \{\partial \ell \circ k_\ell\}_{\ell \in \mL}$. The predictor $\tf:\X \rgta [0,1]$ is $(\mL, \eps_2)$-decision-OI where $\eps_2 = \CalErr(\mW', \tp)$. 
\end{enumerate}
\end{theorem}
\begin{proof}
We first prove Part (1).  Conditioned on $\x = x$, by Equation \eqref{eq:disc-taylor} with $t = c(x)$ we can write 
\begin{align*}
    \E[\ell(\ty, c(\x))|\x =x] - \E[\ell(\sy, c(\x))|\x =x] 
    &= \E[(\tf(\x) - \sy) \partial \ell( c(\x) ) |\x = x].
\end{align*}
Hence taking expectations over $\x$ and absolute values,
\begin{align*}
    \lt|\E[\ell(\ty, c(\x))] - \E[\ell(\sy, c(\x))]\rt| 
    &\leq \max_{c \in \mC}\lt|\E[(\tf(\x) - \sy) \partial \ell( c(\x) )]\rt|.
\end{align*}
The LHS corresponds to hypothesis OI, while the RHS to $\mC'$ multiaccuracy error for $\mC' = \{ \partial \ell \circ c\}$. 

We now consider Part (2). Conditioned on $\x = x$, by Equation \eqref{eq:disc-taylor} with $t = k_\ell(\tp(x))$,
\begin{align*}
    \E[\ell(\ty, k_\ell(\tf(x)))|\x = x] - \E[\ell(\sy, k_\ell(\tf(x)))|\x = x] 
    &= \E[(\tf(x) - \sy) \partial \ell (k_\ell (\tf(x)))|\x = x].
\end{align*}
We now take expectations over $\x$, followed by absolute values to get
\begin{align*}
    \E[\ell(\ty, k_\ell(\tf(x)))] - \E[\ell(\sy, k_\ell(\tf(x)))] 
    &= \E[(\tf(x) - \sy) \partial \ell (k_\ell (\tf(x)))|\x = x].
\end{align*}

The LHS corresponds to loss-OI while the RHS measures the weighted calibration error for $\mW' = \{\partial \ell \circ k_\ell\}_{\ell \in \mL}$. 
\end{proof}

It is easy to see that the characterizations above are tight. For instance if $\MAerr(\mC', \tp)$ is larger than $\eps'$, then there exist a $c, \ell$ pair that distinguishes between $\y^*$ and $\ty$ with advantage $\eps'$. 

 \section{Loss-OI for Generalized Linear Models}
\label{sec:glms}

In this section we study Loss OI and omniprediction in the context of  Generalized Linear Models (GLMs), which are well-studied in machine learning and statistics \cite{Kalai04, KalaiS09, KakadeKKS11, GLMbook, GLMnotes, VKnotes, AuerHW95}.  We give a self-contained description of GLMs in Section \ref{sec:glm-prelims}, where we introduce the family $\Lglm$ of convex losses that arise from GLMs. Our main results about GLMs are the following:
\begin{enumerate}
\item We show that calibrated multiaccuracy implies loss-OI for $\Lglm$ (Theorem \ref{thm:glm-ma-c}).
    
\item We show an information-geometric characterization of loss-OI for strictly convex losses in $\Lglm$, showing an equivalence to a Pythagorean theorem for the associated Bregman divergence (Theorem \ref{thm:loss-oi-pyth}).
    
\item As a partial converse, we show that the solution to the $\ell_1$ regularized GLM loss minimization problem over $\Lin(\mC)$ is multiaccurate for $\mC$ (Theorem \ref{thm:glm-ma}).
\end{enumerate}

\subsection{GLMs and loss OI}
\label{sec:glm-prelims}

We start with a transfer function $g': \R \to \R$, which satisfies some desired properties. 

\begin{definition}
\label{def:transfer}
    Let $\mT$ denote the set of transfer functions $g':\R \to \R$ such that:
\begin{enumerate}
    \item $g'$ is continuous and monotonically increasing.  
    \item $\Im(g') \supseteq  [0,1]$.  
\end{enumerate}
\end{definition}

Some common examples of such functions are the $t^k$ for $k$ odd, ReLU, logistic function and the cumulative density function of a continuous distribution such as the Gaussian or the exponential. 

\begin{definition}
        \label{eq:matching-loss}
Given $g' \in \mT$, define the function $g: \R \to \R$ as
\[ g(t) = \int_0^tg'(s)ds. \]
and the {\em matching loss} $\ell_g$ as
\[ \ell_g(y, t) = g(t) - yt = \int_0^t(g'(s) -y)ds. \]
Let $\Lglm = \{ \ell_g(y,t): g' \in \mT\}$. 
\end{definition}

Note that the derivative of $g$ is indeed $g'$, we have just set the constant of integration so that $g(0) = 0$.  A couple of simple observations about our definitions:
\begin{itemize}
    \item The function $g$ is convex since $g'$ is monotonically increasing. Hence $\ell_g(y,t)$ is a convex function in $t$ for every $y \in \zo$. \item We have $\partial \ell (t) = -t$.
\end{itemize}

A generalized linear model \cite{GLMbook, GLMnotes} with the transfer function $g'$ refers to the convex program obtained by trying to minimize the convex loss  $\ell_g \in \Lglm$ over a linear space of hypotheses of the form $\mH = \Lin(\mC, B)$.
  \begin{align}
        \label{eq:breg3}
            \min_{h \in \Lin(\mC, B)}\E[\ell_g(\y, h(\x))].
\end{align}
Lemma \ref{lem:tech-loss} will show that under certain conditions, this can be interpreted as finding the predictor $g' \circ h$ which minimizes a certain Bregman divergence from the Bayes optimal predictor $p^*$ over $h \in \mH$. 

We will show a loss-OI guarantee for $\Lglm$.
In order to apply the decomposition lemma, we need to analyze the function $k_{\ell_g}$. As usual we can extend the first argument to the range $[0,1]$ as $\ell_g(p,t) = g(t) - pt$. Since this is a convex function of $t$, any point where $g'(t) = p$ so that $\partial \ell_g(p,t) /\partial t =0$ is a global minimum. Since $g'$ is monotonically increasing, the set $\{t: g'(t) = p\}$ is an interval, we define $g'^{-1}(p)$ to be the smallest such point in absolute value. It follows that  $g^{-1}$ is a valid choice for $k_\ell$. We define 
 \[ \mT^{-1} = \{ g'^{-1}:[0,1] \to \R \ \forall  g' \in \mT\} .\] The following theorem results from applying Theorem \ref{thm:dec+model}.

\begin{theorem}
\label{thm:glm-ma-c}
    Let $\mC$ be a bounded hypothesis class.
    If the predictor $\tp$ is $(\mT^{-1}, \alpha_1)$-calibrated, and $(\mC, \alpha_2)$-multiaccurate, then for any $B \geq 1$, it is $(\Lglm, \Lin(\mC, B), \alpha_1 + B\alpha_2)$-loss OI.
\end{theorem}
    
For several common functions such as the identity or the ReLU and its variants, $g^{-1}$ is bounded, so we can derive the desired calibration guarantee from $\alpha$-calibration (via Lemmma \ref{lem:fully} and Theorem \ref{thm:dec+model}). But the function $g'^{-1}$ might not always be bounded. For instance when $g'$ is the sigmoid function, and the matching loss is logistic loss, then $k_\ell(p) = \log(p/(1-p))$ is unbounded. But it is possible to relax the boundedness condition to allow for near-optimal post-processing functions which are bounded. For instance, in the case of the sigmoid transfer function, by truncating $k_\ell$ to be bounded by $O(\log(1/\eps))$ in absolute value, we can get within $\eps$ of the loss achieved by $k_\ell$. We refer the reader to Appendix \ref{app:bounded} for a relaxed notion of loss-OI that still gives omniprediciton, which covers the sigmoid and other transfer functions. 

\eat{
\begin{proof}

    For any $g \in \mG_B$, we have 
    \[\partial \ell_g(t) = \ell(1, t) - \ell(0,t) = (g(t) - t) -g(t) = -t. \]
    
    Next we show that $k_{\ell_g}(p) = f'(p)$.  For $p \in [0,1]$, we wish to find $t =k_\ell(p)$ which minimizes ${\ell_g}(p,t) =g(t) - pt$. 
    This is a convex loss, which is minimized when the derivative vanishes: $g'(t) - p =0$, which can be written as $t = f'(p)$. Hence $k_\ell = f'$ is $D$-Lipshcitz, since $|f''(t)| \leq D$, and so is $\partial \ell \circ k_{\ell_g} = -f'(t)$. Hence if $\tp$ is $\alpha$-smoothly calibrated, then it is $(\partial \ell_g \circ f', L\alpha_1)$-calibrated. 
    
    Further, we have $\partial \ell_g \circ c = -c$. Hence being $(\mC, \alpha)$-multiaccurate implies $(\mL_D, \Lin(\mC, B), B\alpha_2)$-hypothesis OI. We now apply the Decomposition lemma (Lemma \ref{lem:implies}). 
\end{proof}
}

\subsection{Loss-OI for GLMs and Pythagorean theorems}

Let $\mT^* \subseteq \mT$ denote the subset of transfer functions $g' \in \mT$ that are strictly increasing and differentiable.
Let $I = \Im(g')$ be its range, so that $g':\R \to I$ is a bijection. The function $g$ is now strictly convex, and $[0, 1] \subseteq I$.

Define the Legendre dual $f:I \to \R$ of $g$ by
\[ f(v) = \max_{t \in \R} v \cdot t - g(t) \ \ \  \text{for} \ v \in I. \]
For $v^* \in I$, let $t^* \in \R $ be such that $v^* = g'(t^*)$. The existence and uniqueness of such a $t^*$ is guaranteed since $g'$ is a bijection. Since the objective $v^*t -g(t)$ is strictly concave, and its derivative $v^* - g'(t)$ vanishes at $t^*$, $t^*$ must be its unique maximizer. So 
\begin{align}
    f(v^*) &= \max_{t \in \R} g'(t^*)t - g(t) = g'(t^*)t^* - g(t^*)\label{eq:f}
\end{align}
One can also verify the following identities:
\begin{align}
  f'(v^*) &= \frac{d (g'(t^*)t^* - g(t^*))}{d t^*} \frac{d t^*}{dv^*}  = \frac{t^*g''(t^*)}{g''(t^*)} = t^*, \label{eq:f'}\\
    f''(v^*) &= \frac{d t^*}{d v^*} = \fr{g''(t^*)}.\label{eq:f''}
\end{align}
Equation \eqref{eq:f'} implies that for $t \in \R$, $f'(g'(t)) =t$, and for $v \in I$, $g'(f'(v)) =v$.

\paragraph{Bregman divergences:}

The Bregman divergence $D_f: I \times I \rgta \R$ corresponding to $f$ is defined as
\[ D_f(v^*, v) = f(v^*) - f(v) -(v^* - v)f'(v). \]
We say that $f$ is $\lambda$-strictly convex if $f''(v) \geq \lambda$ for $v \in \izo$. For such $f$ we have the inequality
\[ f(v^*) \geq f(v) + (v^* -v)f'(v) + \frac{\lambda}{2}(v^* - v)^2.\]
Hence $D_f(v^*,v) \geq \lambda(v^* -v)^2/2$, and it vanishes iff $v^* =v$. Note that if $g'$ is $L$-Lipschitz, then $g''(x) \leq L$ so by Equation \eqref{eq:f''}, 
\[ f''(x) = \fr{g''(x)} \geq \fr{L}.\]

Consider the following Bregman divergence minimization problem:
\begin{align}
\label{eq:breg1}
    \min_{\tp:\X \rgta I}\E_{\x \sim \mD} [D_f(p^*(\x), \tp(\x))] 
\end{align}
Without restrictions on the structure of $\tp$, the unique minimizer is given by $p^*$. Generalized linear models parameterize $\tf$ in a way that renders the resulting program convex. We consider predictors $\tp$ belonging to the class of {\em generalized linear models} $\{ g' \circ h\}_{h \in \Lin(\mC, B)}$ and solve the program
\begin{align}
\label{eq:breg2}
    \min_{h \in \Lin(\mC, B)}\E_{\x \sim \mD}[D_f(p^*(\x), g' \circ h(\x))] 
\end{align}
The key advantage of this choice of (inverse) link function is that it results in a convex optimization problem, that of minimizing the matching loss function $\ell_g$. The following lemma can be derived from the literature  \cite{Nielsen2, AuerHW95, GLMnotes}, but we are unable to find a precise reference, so we present a proof in Appendix \ref{app:glms}.

\begin{lemma}
\label{lem:tech-loss}
For $g' \in \mT^*$, Program \eqref{eq:breg2} is equivalent to the Program \eqref{eq:breg3}.
\end{lemma}

We relate loss-OI for losses of the form $\ell_g$ to Pythagorean theorems for the Bregman divergence $D_f$. Let $p^*, \tp, p: \X \rgta I$ be  predictors\footnote{Here we will allow predictors taking values in the interval $I$ which contains $[0,1]$}. An exact Pythagorean bound for $(p^*, \tp, p)$ is the statement 
\[ \E[D_f(p^*(\x), p(\x))] = \E[D_f(p^*(\x), \tp(\x))] + \E[D_f(\tp(\x), p(\x))]. \] 
In an approximate bound, the absolute value of the difference of the LHS and RHS is bounded.
In our setting $p^*:\X \to [0,1]$ will be the Bayes optimal predictor, $\tp$ will be calibrated and $\mC$-multiaccurate, while $p = g' \circ h: \X \to I$ belongs to the class of GLMs. The Pythagorean theorem says that minimizing the divergence to $p^*$ for models $p$, is equivalent to minimizing the divergence to $\tp$, which is clearly in the spirit of outcome indistinguishability.  When we take $g'(t) = t$ so that $f(t) = g(t) = t^2/2$, $D_f$ is just the squared Euclidean distance and the Pythagorean theorem has the familiar form of
\[ \E[(p^*(\x) -p(\x))^2] = \E[(p^*(\x) -\tp(\x))^2] + \E[(\tp(\x) -p(\x))^2].\]
Here the statement implies that the error $p^* -\tp$ is {\em orthogonal} to the space spanned by $\tp - p$ over all generalized linear models $p$.

\begin{theorem}
\label{thm:loss-oi-pyth}
    Let $g' \in \mT^*$ and let $f$ be the Legendre dual of $g$. The predictor $\tp$ is  $(\ell_g, \mH, \alpha)$-loss OI iff the following approximate Pythagorean bound holds for every model $h \in \mH$ :
    \begin{align*}
    \lt| \E[D_f(p^*(\x), \tp(\x))] + \E[D_f(\tp(\x), g'(h(\x)))] - \E[D_f(p^*(\x), g'(h(\x)))] \rt| \leq \alpha
    \end{align*}
\end{theorem}
\begin{proof}
Recall that loss-OI corresponds to fooling the distinguishers
 \[ d_{h}(y, \tp(x), x) = \ell_g(y, h(x)) - \ell_g(y, f'(\tp(x)). \]
 For the distribution $\mD$ we have
 \begin{align}
 \label{eq:loi-1}
     \E_{\mD} [\ell_g(\y^*, h(\x)) - \ell_g(\y^*, f'(\tp(\x))] &=  \E[D_f(p^*(\x), g'(h(\x))) - f(p^*(\x)) -  (D_f(p^*(\x), \tp(\x)) - f(p^*(\x)))]\notag\\
     &= \E[D_f(p^*(\x), g'(h(\x))) -  D_f(p^*(\x), \tp(\x))].
 \end{align}
 For the distribution $\mD(\tp)$ we have
  \begin{align}
  \label{eq:loi-2}
     \E_{\mD(\tp)} [\ell_g(\ty, h(\x)) - \ell_g(\ty, f'(\tp(\x))]     &= \E[D_f(\tp(\x), g'(h(\x))) -  D_f(\tp(\x), \tp(\x))] \notag\\
     &= \E[D_f(\tp(\x), g'(h(\x)))].
 \end{align}
 Loss-OI, which asserts that the LHS of Equations \eqref{eq:loi-1} and \eqref{eq:loi-2} are within $\alpha$, is equivalent to
 \begin{align*}
    \lt|\E[D_f(p^*(\x), g'(h(\x))) -  \E[D_f(p^*(\x), \tp(\x))] - \E[D_f(\tp(\x), g'(h(\x)))]\rt| \leq \alpha
\end{align*}     
\end{proof}
\eat{
The reason we call this a Pythagorean bound is that when $\alpha =0$, it reads
\begin{align*}
    \E[D_f(p^*(\x), g'(h(\x)))] = \E[D_f(p^*(\x), \tp(\x))] + \E[D_f(\tp(\x), g'(h(\x)))].
\end{align*}
When we take $f = x^2/2$, $D_f$ is just the squared Euclidean distance and this gives a familiar Pythagorean theorem  for squared loss.
}
\subsection{Multiaccuracy from regularized GLMs}

Finally, we show that multiaccuracy and GLMs are intimately connected. For an appropriately chosen transfer function $g':\R \to [0,1]$, the optimal solution to the convex program of minimizing $\ell_g$ subject to $\ell_1$ regularization yields a multiaccurate predictor. Indeed the infinity norm of the gradient vector corresponds to the multiaccuracy error of the predictor. 

\begin{theorem}
\label{thm:glm-ma}
Fix a transfer function $g' \in \mT$ whose range is  $[0,1]$ and let $\ell_g$ be its matching loss. Let $h^*$ be the optimal solution to the $\ell_1$-regularized loss minimization problem:
    \begin{align}
\min_{h \in \Lin(\mC)}[\ell_g(\y, h(\x))] + \alpha\sum_c|w_c| \ \text{where} \ h(x) = \sum_c w_cc(x)
        \end{align}
    The function $g' \circ h: \X \to [0,1]$ is a predictor and it is $(\mC, \alpha)$-multiaccurate.    
\end{theorem}
\begin{proof}
    For $h(x) = \sum_c w_cc(x)$ define $L(w) = \E[\ell_g(\y, h(\x))]$
    so that $L$ is a convex function of $w$. 
    We can use the chain rule to write
    \begin{align*}
        \frac{\partial L}{\partial w_c} = \E\lt[ \frac{d \ell(\y, t)}{d t}\big\vert_{t = h(\x)}\frac{\partial h(\x)}{\partial w_c} \rt] = \E[(g'(h(\x)) - \y)c(\x)]. 
    \end{align*}
    Let $\sign(t)$ be the sub-gradient of $|t|$, so that when $t=0$, it can take any value in $[-1,1]$. Note that $|\sign(t)| \leq 1$ for all $t$.
    If $w^*$ is the parameter vector of $h^*$, then the (sub)-gradient of the loss vanishing is equivalent to the following equality holding for every $c \in \mC$:
    \begin{align*}
        \E[(g'(h(\x)) - \y)c(\x)] + \alpha\sign(w_c) = 0
    \end{align*}
    Rearranging and taking absolute values,
    \begin{align*}
        \lt|\E[(g'(h(\x)) - \y)c(\x)] \rt|  \leq  \alpha\lt|\sign(w_c)\rt|  \leq \alpha.
    \end{align*}
    Since we assumed $g':\R \to [0,1]$, the function $g' \circ h: \X \to [0,1]$ is a predictor, and it is $(\mC, \alpha)$-multiaccurate by the above inequality.    
\end{proof}

This tells us that multiaccuracy is computationally  easy to achieve, assuming access to a weak agnostic learner for the class $\mC$. In particular, one could use logistic regression with $\ell_1$ regularization or the clipped ReLU transfer with range $[0,1]$ and its associated matching loss.\footnote{The clipped ReLU $g'$is defined by $g'(t) = 0$ for $t \leq 0$, $g'(t) =1$ for $t \geq 1$ and $g'(t) =t$ otherwise. The matching loss is  $\ell_g(y,t) = g(t) -yt$ where  $g(t) = 0$ for $t \leq 0$, $g(t) = t^2/2$ for $t \leq 1$ and $t -1/2$ for $t \geq 1$.}

Using least squares with $\ell_1$ regularization (as in the Lasso algorithm), corresponding to the identity transfer function $g'(t) = t$ will result in an output $g' \circ h$ which is multiaccurate, but which need not be bounded in $[0,1]$. Truncating the output to $[0,1]$ reduces the squared loss, but might loose multiaccuracy. We can run least squares on the residues $\y - \Pi_{[0,1]}(g'(h(\x)))$ to regain multiaccuracy, again reducing the squared loss, but possibly losing boundedness. Alternating between regularized least squares and truncation will converge to a predictor that outputs values in $[0,1]$ and is multi-accurate.  We leave the details to the interested reader. A similar alternating approach is used in the proof of Theorem \ref{thm:alg-map} to achieve the stronger notion of calibrated  multiaccuracy.  

\section{Loss OI for general families of losses}
\label{sec:gen-loss}

In this section, we instantiate the loss-OI framework to derive omniprediction guarantees for more general classes of losses than the convex, Lipschitz losses considered in the work of \cite{omni}. In particular, we explore the effect of relaxing each of those requirements. Our approach is to fix a loss class $\L$, then analyze for any class of hypotheses $\C$, the structure of the class $\set{\partial \ell \circ c}$. Doing so lets us derive loss OI guarantees where the complexity of the weak learning primitive we need grows with the expressiveness of $\L$. We also present results for the $\ell_p$ losses.

\subsection{Arbitrary losses}

    Define the class $\mL_{\mathrm{all}}$ to consist of all loss functions $\ell$ such that $\infnorm{\partial \ell} \leq 1$. We can work with any constant in place of $1$ by rescaling. \footnote{Strictly speaking, we don't require boundedness of $\partial \ell$ over its entire domain,  it suffices if $\partial \ell$ is bounded for $\Im(\mC) \cup \Im(k_\ell)$.} Let $\mC = \{c:\X \rgta \R\}$ be a possibly unbounded hypothesis class. Define the class $\level(\mC)$ to be all functions on the level sets of $\mC$ with range $[-1,1]$. Formally, we define $\level(\mC) = \{f \circ c\}$ where $f:\Im(\mC) \rgta [-1,1]$ and $c \in \mC$. 
\begin{theorem}
\label{thm:general-loss}
For a hypothesis class $\mC = \{c:\X \rgta \R\}$,
if $\tp$ is $\alpha_1$-calibrated and  $(\level(\mC), \alpha_2)$-multiaccurate, then it is $(\mL_{\mathrm{all}},\mC, \alpha_1 + \alpha_2)$-loss OI.
\end{theorem}
\begin{proof}
    By the definition of $\mL_{\mathrm{all}}$, for the weight family $\mW' = \{\partial \ell \circ k_\ell\}$ we have $\infnorm{\mW'} \leq 1$. Hence by Lemma \ref{lem:fully}, if $\tp$ is $\alpha$-calibrated, then $\CalErr(\mW', \tp) \leq \alpha$. By Theorem \ref{thm:dec+model}, if $\tp$ is $\alpha_1$-calibrated, then it satisfies  $(\mL_{\mathrm{all}}, \alpha_1)$-decision OI.

    If $\ell \in \mL$, then $\partial \ell \circ c \in \level(\mC)$, since we assume that $\infnorm{\partial \ell} \leq 1$. Hence by Theorem \ref{thm:dec+model} being $(\level(\mC), \alpha_2)$-multiaccurate implies $(\mL, \mC, \alpha_2)$-hypothesis OI.  
    
    By the decomposition lemma, these conditions together imply  $(\mL, \mC,  \alpha_1 + \alpha_2)$-loss OI.
\end{proof}

\subsubsection{On the complexity of $\level(\mC)$}

In general, the class $\level(\mC)$ might be much more expressive than $\mC$ itself. Since $\mC$-multiaccuracy is known to be equivalent to weak agnostic learning  for the class $\mC$, achieving multiaccuracy for $\level(\mC)$ might be computationally more complex than achieving it for $\mC$. For instance, if $\mC$ contains linear combinations of features $x_i$, then $\level(\mC)$ contains all halfspaces. However in the case where $\mC$ has small range, they might not be too different. For Boolean functions, we can show the following:

\begin{lemma}
\label{lm:boolean}
   If $\mC = \{c: \X \rgta \{0,1\}\}$ consists of Boolean functions, then $(\mC, \alpha)$-multiaccuracy implies $(\level(\mC), 3\alpha)$-multiaccuracy.
\end{lemma}
\begin{proof}
    Take any $f: \{0,1\} \rgta [-1,1]$. Since $\mC$ is Boolean, for $t \in \{0,1\}$ we can write $f(t) = at +b$ where $|a|+ |b| \leq 3$. Hence $f\circ c \in \Lin(\mC, 3)$. We now apply Lemma \ref{lem:linear}. 
\end{proof}

Combining  \Cref{thm:general-loss} and \Cref{lm:boolean}, we get the following corollary:
\begin{corollary}
\label{corollary:boolean}
    Let $\mC = \{c: \X \rgta \{0,1\}\}$ be a class of Boolean functions. If $\tp$ is $\alpha$-calibrated and  $(\mC, \alpha)$-multiaccurate, then it is $(\mL_{\mathrm{all}},\mC, 4\alpha)$-loss OI. 
\end{corollary}

For Boolean hypotheses,  the class $\mL_{\mathrm{all}}$ includes the $0$-$1$ loss and its weighted variants. In this case the omniprediction guarantee is equivalent to agnostic learning. Thus for Boolean functions, calibrated multiaccuracy (which is multiaccuracy and calibration) suffices for agnostic learning. 

Another important class where $\level(\mC)$ is not more complex than $\mC$ is decision trees. When $\mC =\{c: \X \rgta [-1,1]\}$ is the class of decision trees (of bounded size/depth), then $\level(\mC) = \mC$, since given a decision tree $c$, the decision tree where we replace each leaf label $v$ by $f(v)$ computes $f \circ c$ with the same size and depth.

\eat{
\begin{proof} 
$\tp$ is $(\level(\mC),2\alpha)$-multiaccurate by \Cref{lm:boolean}. By \Cref{thm:general-loss}, $\tp$ is $(\mL_{\mathrm{all}},\mC, 3\alpha)$-loss OI.
\end{proof}
}

\subsection{Lipschitz losses}

Under suitable assumptions of Lipschitzness, we can replace $\level(\mC)$ with a simpler class functions. We first define the class of losses we consider.

\begin{definition}
Define $\Lip$ to be the set of loss functions $\ell$ where 
\begin{itemize}
\item $\Im(k_\ell) \subseteq [-1,1]$.
\item On the interval $[-1,1]$, $\partial \ell$ is $1$-Lipschitz and $|\partial \ell (t)| \leq 1$. 
\end{itemize}
\end{definition}

Clearly $\Lip \subseteq \mLa$. 
Let $\mC$ be a bounded hypothesis class. 
For a given $\delta \geq 0$, partition the interval $[-1,1]$ into intervals $\{I^\delta_j\}_{j=1}^{2/\delta}$ of width $\delta$ where $m \leq 2/\delta$. Define the family of functions
\[ \Int(\mC, \delta) = \{\ind{c(x) \in I^\delta_j}\}_{j \in [m], c \in \mC}. \]

\begin{lemma}
\label{lem:int-ma}
   If $\tp$ is $(\Int(\mC,\alpha), \alpha^2)$-multiaccurate, then it is $(\Lip, \mC, 3\alpha)$-hypothesis OI.
\end{lemma}
\begin{proof}
We first show that every function  $\partial \ell \circ c$ can be uniformly approximated by a linear combination of functions from $\Int(\mC, \delta)$. More precisely,
for every $\ell \in \Lip$, there exist constants $l_1,\ldots, l_m \in [-1, 1]$ such that
\begin{align}
    \label{eq:lip-approx}
 \max_{t \in [-1, 1]} \lt| \partial\ell(c(x)) - \sum_{j=1}^m l_j \ind{c(x) \in I_j} \rt| \leq \delta/2.
\end{align}
For each interval $I_j$, let $t_j$ be its midpoint so that $|t - t_j| \leq \delta/2$ for $t \in I_j$. Let $l_j = \partial \ell(t_j) \in [-1,1]$. Since $\partial \ell$ is $1$-Lipschitz on $[-1,1]$, for every $t \in I_j$, $|\partial \ell(t) - l_j| \leq \delta/2$. Equation \eqref{eq:lip-approx} follows by setting $t = c(x)$. 
    
Hence, under $(\Int(\C,\delta),\alpha^2)$-multiaccuracy, it follows that 
    \begin{align*}         
        \lt| \E[(\y - \tp(\x))\partial \ell(c(\x))]  \rt| & \leq \frac{\delta}{2} + 
    \sum_{j=1}^m |l_j|\lt|\E[(\y - \tp(\x)) \Int_{I_j}(c(\x))]\rt|
    \leq \frac{\delta}{2} +\frac{2\alpha^2}{\delta} \leq 3\alpha
    \end{align*}
    by choosing $\delta =\alpha.$
\end{proof}

Since $|k_\ell(p)| \leq 1$, we have $|\partial \ell \circ k_\ell(p)| \leq 1$. Hence the family of weight functions $\mW = \{\partial \ell \circ k_\ell\}_{\ell \in \mL}$ is bounded by $1$. So we can bound $\mC(\mW, \tp) \leq \alpha$ if $\tp$ is $\alpha$-calibrated. Hence the decomposition lemma together with Lemma \ref{lem:int-ma} gives the following claim.

\begin{theorem}
Let $\mC$ be a bounded hypothesis class.  If $\tp$ is $\alpha$-calibrated and  $(\Int(\mC, \alpha), \alpha^2)$-multiaccurate, then it is $(\Lip,\mC, 4\alpha)$-loss OI.
\end{theorem}

\subsection{Low degree losses}

\begin{definition}
    Define $\mL_{d, B}$ to be the set of all loss functions where
    \[ \partial \ell (t) = \sum_{i=0}^d b_j t^j, \ \text{where} \ \sum_j|b_j| \leq B. \]
    For a hypothesis class $\mC:\X \rgta \R$, let $\mC^d$ denote the hypotheses class $\{c(x)^j\}_{c \in \mC, j \in [d]}$.
\end{definition}
We refer to such losses as low-degree losses. We have already seen that $\Lglm \in \mL_{1,1}$, since $\partial \ell_g(t) = -t$.  The following lemma is an immediate consequence of our definitions and the decomposition lemma.

\begin{lemma}
\label{lem:low-degree}
    Let $d \in \Z^+$ and $B \in \R^+$ If $\tp$ is $\alpha$-calibrated and  $(\mC^d, \alpha/B)$-multiaccurate, then it is $(\mL_{d, B},\mC, 2\alpha)$-loss OI.
\end{lemma}

The $\ell_p$ losses naturally yield low-degree losses. 
\begin{itemize}
\item For unbounded $\mC$ and $p$ even, for $\ell(y, t) = (y -t)^p/p$, we have
    \[ \partial \ell_p (t) = \fr{p}((1-t)^p -t^p). \]
    Since $p$ is even, the $t^p$ term cancels and we have $\ell_p \in \mL_{p-1, 2^p}$.

\item     The same expression for $\partial \ell_p$ also holds for odd $p$ if $t$ is bounded to lie in the range $[0,1]$. In this setting, $\partial \ell_p$ is a degree $p$ polynomial for $t \in [0,1]$. In particular, when $p=1$, $\partial \ell_1 = 1 -2t$ for $t \in [0,1]$. If we restrict $\mC$ to be bounded in the range $[0,1]$, then we only need a bound for $t \in [0,1]$.
\end{itemize}

    Thus by Lemma \ref{lem:low-degree}, in these settings, we get loss OI for the $\ell_p$ losses from calibration and multiaccuracy for $\mC^p$. 

\eat{
\paragraph{To be proved:}

As an application, we can show the following result, which says that calibration toegther with multiaccuracy for thresholds of $\mC$ implies omniprediction. For a class $\mC:\{c:\X \rgta \R^d\}$, let 
\[ \Thr(\mC) = \{\ind{c(x) \geq t}_{c \in \mC, t \in \R^d}.\]
Recall that $\Lip{1}$ is the set of all Lipschitz loss functions.

\begin{lemma}
    If $\tf$ is $(\eps, \delta)$ calibrated, it satisfies $(\Lip{1}, \eps + \delta)$-decision OI. If $\tf$ is $(\Thr(\mC), \eps)$-multiaccurate, then it satifies $(\Lip{1}, \mC, \eps)$-hypothesisOI.
\end{lemma}
 
\begin{corollary}
    If $\tf$ is $\eps, \delta$-calibrated and $(\Thr(\mC), \eps)$-multiaccurate, then it is $(\Lip{1}, O(\eps + \delta))$-loss OI.
\end{corollary}
}

 \section{Efficient algorithms for $\MAp$}
\label{sec:algo}

Let $\mC = \{c: \X \rgta [-1,1]\}$ which contains the constant $1$ function and is closed under negation. We recall the following classes
\begin{itemize}
    \item Let $\MAcc(\alpha)$ denote the set of predictors that are $(\mC,\alpha)$-multiaccurate.
    \item Let $\MAp(\alpha)$ denote the set of predictors that are $\alpha$-calibrated and $(\mC, \alpha)$-multiaccurate. 
    \item Let $\MCal(\alpha)$ denote the set of predictors that are $(\mC, \alpha)$-multicalibrated. 
\end{itemize}
Then we have $\MAcc(\alpha) \supseteq \MAp(\alpha) \supseteq \MCal(\alpha)$. In this section we will give an efficient algorithm to compute a predictor in $\MAp(\alpha)$. Like with algorithms for $\MAcc$ and $\MCal$, we will assume oracle access to a weak agnostic learner for $\mC$. The complexity of the algorithm hinges on the number of calls made to the weak agnostic learner. The main takeaway from this  section is the worst-case number of calls needed for $\MAp$ is similar to that for $\MA$ and lower than what is needed for $\MCal$.

\paragraph{Weak agnostic learning.}  We first define the notions of a weak agnostic learner and a discrete predictor which will be needed for our algorithm.

\begin{definition}
\label{def:weak}
    Let $\mD_\X$ be a distribution over $\X$. A $(\rho, \sigma)$-weak learner $\WL$ for $\mC$ under $\mD_\X$ is an algorithm $\WL$, whose input is specified by a function $f: \X \rgta [-1,1]$.
    \begin{itemize}
        \item The algorithm is given sample access to $f$ via samples $(\x, \z) \in \X \times \pmo$ where $\x \sim \mD_\X$, and $\E[\z|\x = x] = f(x)$. 
        \item If there exists $c \in \mC$ such that $\E_\mD[c(\x)f(\x)] \geq \rho$,
        then $\WL(f, \rho) = c' \in \mC$ such that $\E_{\mD}[c'(\x)f(\x))] \geq \sigma$.
        \item  If no such $c$ exists, the weak learner returns $\WL(f) = \bot$.
    \end{itemize}
\end{definition}

Some observations about our definition:
\begin{itemize}
    \item A non-proper learner is allowed to return hypothesis from a class $\mC'$ which is different from $\mC$. Our analysis goes through unchanged in this setting, we set $\mC' =\mC$ for simplicity. Typically, the weak learner will only succeed with probability $1 - \delta$. However, standard amplification allows us to make $\delta$ small at an added cost of $\log(1/\delta)$, so we ignore this failure probability for simplicity. Assuming that $\mC$ is closed under negation is also a notational convenience, it lets us suppress the sign of the correlation in the updates.

    \item  In our algorithms, $f(x)$ will take the form $p^*(x) - p_t(x)$ where $p^*$ is the Bayes optimal predictor and $p_t: \X \rgta [0,1]$ is our current hypothesis predictor. 
    The weak agnostic learner requires sample access to $f$, which can be simulated via a standard trick in the literature on distribution-specific agnostic boosting \cite{kk09, feldman2009distribution}. 
    Note that $p^*(x) - p_t(x) \in [-1,1]$. In order to simulate sample access to $f$, we draw a sample $(\x, \y) \sim \mD$. Then we generate $\z \in \pmo$ so that $\E[\z] = \y  -p_t(\x)$. Since $\y - p_t(\x) \in [-1,1]$, this uniquely specifies the distribution of  $\z$. Moreover
    \[ \E[\z|\x = x] = \E[\y|\x =x] - p_t(x) = p^*(x) -p_t(x) = f(x). \]
    Alternatively, some weak learners may accept real-valued labels; in this case, we can use $z = y - p_t(x)$ to label $x \sim \D_\X$.
\end{itemize}

We define the following norms over the space of predictors $p: \X \rgta [0,1]$:
\begin{align*}
    l_1(p_1,p_2) &= \E[|p_1(\x) - p_2(\x)|]\\
    l_2(p_1,p_2) &= \E[(p_1(\x) - p_2(\x))^2]^{1/2}\\
    l_\infty(p_1,p_2) &= \max_{x \in \X} |p_1(\x) - p_2(\x)|
\end{align*}
and observe that $l_1(p_1,p_2) \leq l_2(p_1,p_2) \leq l_\infty(p_1, p_2)$. Our algorithms will use the potential function 
\[ l_2(p^*, p)^2 = \E[(p^*(\x) - p_2(\x))^2]. \]

We record the following technical lemma showing that multiaccuracy error and squared loss are robust under perturbations of the predictor. The proof is in Section \ref{app:algo}.

\begin{lemma}
\label{lem:ell1}
For any predictors $p_1, p_2$ such that $l_1(p_1,p_2) \leq \delta$,
\begin{align}
    \label{eq:first-sq}
        \lt| l_2(p^*, p_1)^2 - l_2(p^*, p_2)^2\rt| &\leq 2\delta.
    \end{align}
Further, if $p_1$ is $(\mC, \alpha)$-multiaccurate, then $p_2$ is $(\mC, \alpha + \delta)$-multiaccurate.
\end{lemma}

\subsection{Discrete predictors and Calibration}

We say that a predictor $p: \X \rgta [0,1]$ is $\delta$-discrete if its predictions are integer multiples of $\delta$. For every predictor $p$, we associate it with a $\delta$-discrete predictor $\pone$ as follows. We partition the interval $[0,1]$ into $m = \lceil 1/2\delta \rceil$ intervals $\{I_1, \ldots, I_m\}$ of width $2\delta$ each,  where $I_j = [(2j -2) \delta, 2j\delta)$. For $x \in I_j$, we define 
\[ \pone(x) = (2j -1)\delta, \ \ \ptwo(x) =  \E[\y|p(\x) \in I_j] \]

Some observation about these predictors:
\begin{enumerate}
    \item The predictor $\pone$ is $\delta$-discrete and $l_\infty(p, \pone) \leq \delta$. Hence its squared loss and its multiaccuracy error are not much greater than that of $p$ (by Lemma \ref{lem:ell1}). But it need not be calibrated.

    \item We can view $\ptwo$ as the result of recalibrating $\pone$, so $\ECE(\ptwo) = 0$.  Since $\ptwo$ is obtained by calibrating $\pone$, its squared loss is less than that of $\pone$ (since the mean is the constant value that minimizes the squared error), and even less than that of $p$ for suitable parameter settings. 
    
    \item The predictor $\ptwo$ is not efficiently computable since it is defined in terms of the expectation of $\y$ under $\mD$. In Lemma \ref{lem:eff-cal}, we give an efficient  approximation $\pth$ to it (via random sampling), at the cost of a small increase in $\ECE$.  
\end{enumerate}

We formalize observation (2) below,  relating the reduction in squared loss to $\ECE(\pone)$.  

\begin{lemma}
\label{lem:err-red}
For the predictors $\pone, \ptwo$ defined above, 
\begin{align}
\label{eq:third-sq}
        l_2(p^*, \pone)^2 - l_2(p^*, \ptwo)^2 & \geq \ECE(\pone)^2.
\end{align}
\end{lemma}
\begin{proof}
\eat{
We first show that
\begin{align}
    \label{eq:second-sq}
    \E[(p^*(\x) - \pone(\x))^2] - \E[(p^*(\x) - \ptwo(\x))^2] & = \E[(\ptwo(\x) - \pone(\x))^2.
\end{align}}
We write the LHS of Equation \eqref{eq:third-sq} as
\begin{align*}
    \E[(p^*(\x) - \pone(\x)^2] - \E[(p^*(\x) - \ptwo(\x)^2] &= \E[(\ptwo(\x) - \pone(\x))(2p^*(\x) - \pone(\x) - \ptwo(x))]
\end{align*}
We consider the distribution on intervals induced by choosing $\x \sim \mD$ and $\mb{I_j} \ni p(\x)$. Since $\pone$ and $\ptwo$ are constant for each interval $I_j$, we can write $\pone(I_j)$ and $\ptwo(I_j)$ for their values in this interval without ambiguity. Hence by first taking expectations over $\mb{I_j}$ and then $p(\x) \in \mb{I_j}$
\begin{align*}
    \E[(\ptwo(\x) - \pone(\x))(2p^*(\x) - \pone(\x) - \ptwo(x))]
    &= \E_{\mb{I_j}}\lt[(\ptwo(\mb{I_j}) - \pone(\mb{I_j}))\E_{\x|p(\x) \in I_j}[2p^*(\x) - \pone(\mb{I_j}) - \ptwo(\mb{I_j})] \rt]\\
    &= \E_{\mb{I_j}}\lt[(\ptwo(\mb{I_j}) - \pone(\mb{I_j}))^2\rt]\\
    &= \E_{\x \sim \mD}\lt[(\ptwo(\x) - \pone(\x))^2\rt].
\end{align*}
where the penultimate line uses $\E[p^*(\x)|\x \in I_j] = \ptwo(I_j)$. 
    
Since $\pone, \ptwo$ are both constant one each interval $I_j$, we have
    \begin{align*}
        \E_{\x \sim \mD}\lt[(\ptwo(\x) - \pone(\x))^2\rt] &= \E_{\mb{I_j}}\lt[\E[\ptwo(\x)  - \pone(\x))|p(\x) \in I_j]^2\rt]\\
        &= \E_{\mb{I_j}}\lt[\E[\y - \pone(\x))|p(\x) \in I_j]^2\rt]\\
        &\geq \E_{\mb{I_j}}\lt[ \lt|\E[\y - \pone(\x))|p(\x) \in I_j]\rt|\rt]^2\\
        &= \ECE(\pone)^2
    \end{align*}
where the first inequality uses the convexity of $x^2$. \eat{We can rewrite Equation \eqref{eq:first-sq} as
\[ \E[(p^*(\x) - p(\x))^2] - \E[(p^*(\x) - \pone(\x))^2] \geq -2\delta.\]
Adding these inequalities gives the desired claim.}
\end{proof}

\begin{lemma}
\label{lem:eff-cal}
    Let $\mu, \delta \in [0,1]$.  Given acess to a predictor $p$ and random samples from $\mD$,
    \begin{itemize} 
    \item There exists an algorithm $\estECE(p, \mu)$ which returns an estimate  of $\ECE(\pone)$ within additive error $\mu$. The algorithm runs in time and sample complexity $\tilde{O}(1/(\delta\mu^3))$. 
    \item There exists an algorithm $\calg(p, \delta)$ which returns a predictor $\pth$ which has $\ECE(\pth) \leq \delta$ and $l_1(\ptwo, \pth) \leq \delta$. The algorithm has time and sample complexity $\tilde{O}(1/\delta^4)$.
    \end{itemize}
\end{lemma}
For both algorithms, the stated guarantees hold with a failure probability can be made arbitrarily small by standard amplification. For simplicity, we have omitted this from the statement. The proof of this claim is through standard sampling arguments and use of Chernoff bounds. We record the following corollary, which follows from Lemmas \ref{lem:ell1}, \ref{lem:err-red} and \ref{lem:eff-cal}. The proofs are in Appendix \ref{app:algo}. 

\begin{corollary}
\label{cor:err-red}
For the predictors $p, \pone, \pth$ defined above, 
\begin{align}
\label{eq:fourth-sq}
        l_2(p^*,p)^2- l_2(p^*,\pth)^2  & \geq \ECE(\pone)^2 - 4\delta.
\end{align}
\end{corollary}

\subsection{Multiaccuracy and Calibrated Multiaccuracy}

We now state and analyze the algorithm for achieving calibrated multiaccuracy.
To begin, we recall the algorithm of \cite{hkrr2018} for learning multiaccurate predictors.
We present a formulation that will be useful for our main algorithm (Algorithm \ref{alg:map}).
The algorithm $\MA$ is given a predictor $p_0$ as input.
It returns as output a predictor $\tp \in \MC(\alpha)$, such that the squared distance to $p^*$ only decreases. 
We use $\Pi(h)$ to denote the clip operator which takes a possibly real valued function $h:\X \rgta \R$ and truncates any values that lie outside $[0,1]$ to the closest value in $[0,1]$.

\begin{algorithm}
\label{alg:wmc}
\caption{$\mathsf{MA}$}
\textbf{Input:}  Predictor $p_0: \X \rgta [0,1]$\\
Error parameter $\alpha \in [0,1]$. \\
Oracle access to a $(\rho, \sigma)$ Weak learner $\WL$ for $\mC$ under $\mD_\X$ where $\alpha \geq \rho$ .\\
\textbf{Output:}  Predictor $p_T$.
\begin{algorithmic}
\STATE $t \gets 0$
\STATE $ma \gets \mathsf{false}$
\WHILE{$\neg ma$}
    \STATE $c_{t+1} \gets \WL(p^* -p_t)$. 
    \IF{$c_{t+1} = \bot$}
        \STATE $ma \gets \mathsf{true}$ 
    \ELSE
        \STATE $h_{t+1} \gets p_t + \sigma c_{t+1}$.
        \STATE $p_{t+1} \gets \Pi(h_{t+1})$.
        \STATE $t \gets t+1$.
    \ENDIF
\ENDWHILE
\RETURN $p_t$.
\end{algorithmic}
\end{algorithm}

The counter $t$ tracks the number of non-trivial updates made to the predictor. If we return at $t =T$, then the total number of calls to $\WL$ is $T+1$, where the last call does not yield a non-trivial update. We present the proof in Appendix \ref{app:algo} for completeness.

\begin{lemma}\cite{hkrr2018}
\label{lem:hkrr}
    Assume that algorithm $\MA(p_0, \alpha)$ returns the predictor $p_T$ where $T \geq 0$. Then $p_T \in \MA(\alpha)$ and
    \[ l_2(p^*, p_0)^2 -  l_2(p^*, p_T)^2 \geq T\sigma^2.\]
\end{lemma}

We now present our algorithm for finding a predictor in $\MAp(\alpha)$. We will set the discretization $\delta$ to be small compared to $\alpha$ ($\alpha^2/C$ for some constant $C$). The algorithm may be viewed as starting with an arbitrary predictor $p_0$ and then running the following steps:
\begin{enumerate}
    \item We set $p_{t+1} = \malg(p_t)$ to get a predictor that is mulitaccurate.  
    \item We estimate the calibration error $\ECE(\pone_{t+1})$ using $\estECE$. 
    \begin{enumerate}
        \item If the calibration error is large, we recalibrate it to $\pth_{t+1}$ using $\calg$ so that the calibration error drops to $\delta$ and repeat the loop.
        \item  Else, we return the predictor $\pone_{t+1}$.
        \end{enumerate}
\end{enumerate}

When we terminate, both multiaccuracy and calibration are achieved. Both steps reduce the potential function $l_2(p^*, p_t)^2$, (for suitable choices of parameters) which allow us to bound the overall number of iterations.

\begin{algorithm}
\label{alg:map}
\caption{$\mathsf{calMA}$}
\textbf{Input:}  Predictor $p_0: \X \rgta [0,1]$\\
Error parameter $\alpha \in [0,1]$. \\
Oracle access to a $(\rho, \sigma)$-Weak learner $\WL$ for $\mC$ under $\mD_\X$ where $\alpha - \alpha^2/32 \geq \rho$.\\
\textbf{Output:}  Predictor $q_T$.\\
\begin{algorithmic}
\STATE $\delta \gets \alpha^2/32$.
\STATE $\mu \gets \alpha/4$
\STATE $q_0 \gets p_0$
\STATE $ma \gets \mathsf{false}$ 
\STATE $t \gets 0$
\WHILE{$\neg ma$}
    \STATE $t \gets t+1$
    \STATE $p_{t} \gets  \malg(q_{t-1}, \alpha - \delta)$
    \IF{$\estECE(\pone_{t}, \mu) > 3\alpha/4$}
        \STATE $q_{t} \gets \calg(p_{t}, \delta)$
    \ELSE
        \STATE $q_{t} \gets \pone_{t}$
        \STATE $ma \gets \mathsf{true}$
    \ENDIF
    
\ENDWHILE
\STATE return $q_t$ 
\end{algorithmic}
\end{algorithm}

Some observations about the execution of Algorithm \ref{alg:map}: the counter $t$ tracks the number of executions of the while loop.  Assume that we return at $t = T$. The updates $p_t$ to $q_t$ are made by $\calg$, except the last update where $q_T = \pone_T$.

\begin{theorem}
\label{thm:alg-map}
    For $\alpha > 0$, let $\rho, \sigma, \delta$ be as given in Algorithm \ref{alg:map}, and let $q_T$ be the predictor returned.  
    \begin{enumerate}
        \item $q_T \in \MAp(\alpha)$ and it is $\delta$-discrete. 
        \item The number of iterations of the while loop is bounded by $O(1/\alpha^2)$. 
        \item The total number of calls to $\WL$ is bounded by $O(1/\sigma^2)$. 
    \end{enumerate}
\end{theorem}
\begin{proof}
    We have $q_T = \pone_{T}$, hence it is $\delta$-discrete. The predictor $p_{T}$ is $(\mC, \alpha - \delta)$-multiaccurate, since it is returned by a call to $\malg$. By Lemma  \ref{lem:ell1}, $q_T$ is $(\mC, \alpha)$-multiaccurate. This proves claim (1).

    Assume that when we set $p_t = \malg(q_{t-1}, \alpha - \delta)$, this results in $S_t \geq 0$ calls to the weak learner $\WL$. The by Lemma \ref{lem:hkrr},
    \begin{align} 
    \label{eq:prog1}
        l_2(p^*, q_{t-1})^2 - l_2(p^*, p_{t})^2 \geq S_t\sigma^2.
    \end{align}
    In every iteration of the while loop except the last, we have $\estECE(\pone_t, \mu) \geq 3\alpha/4$. By Lemma \ref{lem:eff-cal}, this means that
    \[ \ECE(\pone_t) \geq \frac{3\alpha}{4} - \mu = \frac{\alpha}{2}. \]
    Since $q_t = \pth_t$, applying Corollary \ref{cor:err-red}, we have
    \begin{align}  
    \label{eq:prog2}
        l_2(p^*, p_t)^2 - l_2(p^*, q_t)^2  \geq \ECE(\pone_t)^2 - 4\delta \geq \frac{\alpha^2}{8}.
    \end{align}
    Adding Equations \eqref{eq:prog1} and \eqref{eq:prog2}, for $t \in \{1, \ldots, T-1\}$,
    \[ l_2(p^*,q_{t-1})^2 - l_2(p^*,q_t)^2 \geq \frac{\alpha^2}{8}. \]
    Summing this over all $t$, 
    \[ l_2(p^*,q_0)^2 - l_2(p^*,q_{T-1})^2 \geq (T-1)\frac{\alpha^2}{8}. \]
    Since $q_0 = p_0$ and $l_2(p^*,q_{T-1})^2 \geq 0$, we have
    \[ T \leq 1 + \frac{8}{\alpha^2}{l_2(p^*,p_0)^2} = O(1/\alpha^2). \]
    
    To bound the number of calls to the weak learner, we sum Equation \eqref{eq:prog1} over all $t \in [T]$, and Equation \eqref{eq:prog2} over all $t \leq T -1$ to get
    \[ l_2(p^*,p_T)^2 - l_2(p^*, p_0)^2 \geq \sum_tS_t \sigma^2 + (T-1)\frac{\alpha^2}{8}. \]
    This implies that 
    \[ \sum_t S_t \leq 1/\sigma^2 \]
    Since the number of calls to the weak learner in loop $t$ is bounded by $S_t +1$, we bound the number of calls by
    \[ \sum_t (S_t + 1) \leq \fr{\sigma^2} + T = O(1/\sigma^2) \] 
    where we bound $T = O(1/\sigma^2)$, since $T= O(1/\alpha^2)$  and $\alpha \geq \rho \geq \sigma$.
\end{proof}

A quick remark about sample complexity: in each iteration of the while loop, we use fresh samples in order to ensure that the data and the current hypotheses are independent. This results in an $O(1/\alpha^2)$ sample overhead. It might be possible to improve the sample complexity using adaptive data analysis techniques as in \cite{hkrr2018}, we leave this open.  A similar issue arose in the original analysis of the Isotron algorithm \cite{KalaiS09}, this was later remedied in the work of \cite{KakadeKKS11}.

\subsection{Complexity Comparison}

In this section we compare and contrast the complexity of algorithms for computing a predictor in $\MA, \MAp$ and $\MC$. For an expressive hypothesis class $\mC$, the running time is likely to be dominated by the calls to the $(\rho, \sigma)$ weak learner $\WL$. We compare the number of oracle calls needed for computing a predictor in each of these classes. We emphasize that this is a comparison between the best known upper bounds. For multiaccuracy, we use the \cite{hkrr2018} algorithm as analyzed in Lemma \ref{lem:hkrr}. For multicalibration, we use the analysis of the algorithm from \cite[Section 9]{omni}, which is derived from the boosting by branching programs algorithm by \cite{MansourM2002}.

\begin{itemize}
    \item For $\MA(\alpha)$, the number of calls made by the algorithm of \cite{hkrr2018} is bounded by $O(1/\sigma^2)$. We require $\alpha \geq \rho$.  
    \item For $\MAp(\alpha)$, the number of calls made by Algorithm \ref{alg:map} bounded by $O(1/\sigma^2)$. We require $\alpha - \alpha^2/32 \geq \rho$. 
    \item For $\MC(\alpha)$, the number of calls made by the algorithm of \cite{omni} is bounded by $O(1/\alpha^2\sigma^4)$. The weak learning assumption required is also somewhat stronger, see Appendix \ref{app:omni} for a detailed discussion. For simplicity, one could say that they require a $(\rho, \sigma)$-weak learner where $\alpha/2 \geq \rho$ under marginal distributions on $\X$ that are different from $\mD_\X$.\footnote{They are obtained by conditioning $\mD_\X$ on states with probability as small as $O(\alpha^2\sigma^2)$.}
\end{itemize}

The comparison above shows that $\MA$ and $\MAp$ have similar complexities in terms of the number of calls to the weak learner. The number of calls required for $\MC$ is significantly larger. 

Perhaps more importantly, the algorithm for $\MAp$ is easy to implement in Python using standard packages for regression and calibration, with some simple additional logic. In contrast, the logic to implement boosting by branching programs is non-trivial, and is not implemented by any standard python libraries to our knowledge \cite{gopalan2021multicalibrated}.
In practice, the additional complexity of full multicalibration manifests primarily in terms of the samples needed to prevent overfitting. The work of \cite{gopalan2021multicalibrated} found that the sample complexity often rendered to algorithm impractical on medium sized real-world datasets. Even if we treat parameters like $\alpha$ and $\sigma$ as reasonably large constants, the data requirement of $\MAp$ compared to $\MC$ could easily reduce by 100-fold.

\subsection{Calibrated multiaccuracy requires non-linear models}

Finally, we discuss the hypothesis class we use to fit a calibrated multiaccurate predictor.
Because of the nature of the additive updates, the \cite{hkrr2018} algorithm for multiaccuracy returns a linear model $p \in \Lin(\C,1/\alpha)$.
Note, however, that the $\MAp$ algorithm does not return such a model.
In particular, the recalibration step introduces a nonlinearity, where we must condition on the value of the prediction $p_t$.
The benefits of linearity arise in the simplicity of working with linear models, but also in the sample complexity.
The easiest way to bound the sample complexity of achieving calibrated multiaccuracy is to take a fresh sample for each iteration of Algorithm~\ref{alg:map}, which loses a polynomial factor in $1/\alpha$.
The sample complexity of multiaccuracy, however, can be bounded more tightly using straightforward uniform convergence arguments.

We may wonder if moving outside the class of linear models, or a slight generalization, is really necessary.
Instead, we might hope that algorithms like the Isotron \cite{KalaiS09,KakadeKKS11} may achieve calibrated multiaccuracy.
The Isotron returns a so-called single index model (SIM) of the form
\begin{gather*}
    p(x) = u\left( \sum_{c \in \C} \lambda_c \cdot c(x) \right)
\end{gather*}
where $u:\R \to [0,1]$ is any monotonic nondecreasing function.
The Isotron algorithm is similar to Algorithm~\ref{alg:map}, switching between updating $\set{\lambda_c}$ based on the residuals given the current predictions and choosing $u$ to recalibrate predictions. However, a crucial difference is that there the recalibration uses isotonic regression. This guarantees a calibrated predictor, but is not guaranteed to reduce the squared error. Given the update rule---as in our algorithm---if the procedure terminates, then multiaccuracy and calibration are guaranteed.
\cite{KalaiS09}, however, only establish convergence in the well-specified setting; in the agnostic setting that we study, it is not clear whether the Isotron is guaranteed to terminate.

Here, we show that the Isotron algorithm might not converge in the agnostic setting.
Concretely, we construct a distribution and class $\C$, such that no SIM over the class $\C$ can achieve calibrated multiaccuracy.
This simple construction shows the necessity of moving outside the class $\Lin(\C,B)$, as in Algorithm~\ref{alg:map}.

\begin{lemma}[Informal]
No agnostic learning algorithm that returns a SIM $p(x) = u\left(\sum_\C \lambda_c \cdot c(x)\right)$ can guarantee $p \in \MAp(\alpha)$ for any $\alpha < 1/10$.
\end{lemma}
\begin{proof}
We exhibit a distribution over the $2$-dimensional boolean cube and a collection of functions $\C$ such that calibrated multiaccuracy cannot be  achieved by any SIM model.
For simplicity, we show the violation for $\alpha = 0.1$
This example can be generalized to any dimension and approximate calibrated multiaccuracy.

Let $\X = \set{00,01,10,11}$, and take the class $\C = \set{x_i = 1, x_i = 0 : i \in [2]}$ to be the four subcubes.
Suppose that Bayes optimal probabilities on each $x \in \X$ are given as
\begin{align*}
    \ps_{00} = 0&&\ps_{01} = 1/2&&\ps_{10} = 1&&\ps_{11} = 0
\end{align*}
and take $\D_\X$ to be uniform over $\X$.

Consider any predictor $p$ that satisfies multiaccuracy and calibration.
The multiaccuracy constraints can be written as:
\begin{gather*}
    x_0 = 0:\qquad 1/2\cdot(p_{00} + p_{01}) = 1/4\\
    x_0 = 1:\qquad 1/2\cdot(p_{10} + p_{11}) = 1/2\\
    x_1 = 0:\qquad 1/2\cdot(p_{00} + p_{10}) = 1/2\\
    x_1 = 1:\qquad 1/2\cdot(p_{01} + p_{11}) = 1/4
\end{gather*}
By the first and final equations, we see that multiaccuracy implies $p_{00} = p_{11}$.
Further, the equations imply that $p_{01} = 1/2-p_{00}$ and $p_{10} = 1-p_{00}$.

Next, we consider the calibration constraints.
To do this, we must determine the level sets of $p$.
We argue that $p_{00} = p_{11}$ but are not equal to either $p_{01}$ or $p_{10}$.
In particular, if we set $p_{00}$ such that $p_{01} = 1/2-p_{00} = p_{00} = 1/4$, or $p_{10} = 1-p_{00} = p_{00} = 1/2$, then we violate calibration.
The expectations over the level set in each of these cases is $1/6$ and $1/3$, respectively, violating the calibration constraint by a constant.

In the alternative case, the level sets are $\set{p_{00} = p_{11}}$,$\set{p_{01}}$, and $\set{p_{10}}$.
Under calibration, $p_{00} = p_{11}$ must equal $0$, by the fact that the true expectation over these sets is $0$.
In fact, this implies that $p$ must be within $\ell_1$ distance $\alpha$ from $\ps$.

Finally, we note that any SIM computes a unate function, that is monotone according to some orientation of each $\set{x_i}$.
The true function $\ps$, however, is not close to unate.
Thus, no SIM can be close to $\ps$ and, by the above analysis, not SIM can achieve calibrated multiaccuracy.
\end{proof}

 \section{Experiments}
\label{sec:exp}

As a proof of concept we implemented a naive version of  algorithm $\mathsf{calMA}$~\eqref{alg:map}, where the weak learner is instantiated by running linear regression with square loss, and the calibration is instantiated by running isotonic regression. 
Specifically, given a set of features $x_1,\ldots,x_n$, and a label $y$, the implementation uses as base classifiers the set of linear functions over the features. Multiaccuracy is obtained since linear regression minimizes square loss. The calibration phase of the algorithm is instantiated by running isotonic regression with a freshly sampled calibration set. Both linear regression and isotonic regression are part of the sklearn Python library, thus the algorithm is remarkably simple to implement and consists of less than 100 lines of Python.

\paragraph{Data:}
The distribution for the 0 label is a mixture of $s \in \{2,4\}$ well separated Gaussian distributions over $d \in \{2,4,10\}$ dimensions. The distribution for the 1 label is the same as the 0 label, but shifted by a unit vector. As the dimensionality $d$ increases, the classification task becomes easier for a linear predictor, and thus we test the algorithm across a range of loss values. While the distributions are simple, they still pose a challenge for simple predictors and demonstrate well the strength of our approach. See Figure~\ref{fig:ground_truth} for an example in $2$ dimensions, where white dots represent points labeled $0$ and green dots are points labeled $1$.
As the data are synthetic, sample complexity is not an issue. We generated $3000$ points for the regression and $1000$ points for the calibration.

\paragraph{Metrics:}
We measured the loss $\mathsf{calMA}$ suffers and compared it to linear regression with $\ell_1$, $\ell_2$, exponential  loss and log loss, using the correct link function for each loss function. Linear regression with various loss functions is implemented using Python's scipy.optimize.minimize package.

\paragraph{Results:}
Results are summarized in the tables below. It is clear that $\MAp$ competes remarkably well with the optimal linear predictor across all tested loss functions, occasionally performing even better than the optimal linear predictor for the loss function. As a sanity check in the first table we tested the ``omniprediction'' of the simple $\ell_2$ predictor and indeed found it failed to produce the correct result of the $\ell_1$ case and log loss. An interesting example in $2$ dimensions is visualized in the figures below. In Figure~\ref{fig:ground_truth} we see the ground truth classification of labels. Linear regression on its own cannot create multiple clusters as can be seen in Figure~\ref{fig:linear_regression}. However, $\MAp$ uses multiple linear classifiers and manages to identify the structure of the clusters, as seen in Figure~\ref{fig:calMA}.

\begin{table}[h!]
\begin{center}
\begin{tabular}{||c c c c c||} 
 \hline
 Algorithm & $\ell_2$ & $\ell_1$ & Exp & Log-loss \\ [0.5ex] 
 \hline\hline
 Optimal & 0.21 & 0.35 & 1.54 & 0.61 \\ 
 \hline
 $\MAp$ & 0.20 & 0.32 & 1.65 & 0.635 \\
 \hline
 Linear Regression & 0.21 & 0.43 & 1.61 & 1.22 \\
 \hline
\end{tabular}
\end{center}
\caption{$s = 2, d=2$. Number of iterations is $4$}
\end{table}

\begin{table}[h!]
\begin{center}
\begin{tabular}{||c c c c c||} 
 \hline
 Algorithm & $\ell_2$ & $\ell_1$ & Exp & Log-loss \\ [0.5ex] 
 \hline\hline
 Optimal & 0.18 & 0.28 & 1.51 & 0.57 \\ 
 \hline
 $\MAp$ & 0.18 & 0.28 & 1.53 & 0.58 \\
 \hline
\end{tabular}
\end{center}
\caption{$s = 4, d=4$. Number of iterations is $5$}
\end{table}

\begin{table}[h!]
\begin{center}
\begin{tabular}{||c c c c c||} 
 \hline
 Algorithm & $\ell_2$ & $\ell_1$ & Exp & Log-loss \\ [0.5ex] 
 \hline\hline
 Optimal & 0.08 & 0.07 & 1.55 & 0.57 \\ 
 \hline
 $\MAp$ & 0.06 & 0.08 & 1.13 & 0.22 \\
 \hline
\end{tabular}
\end{center}
\caption{$s = 4, d=10$. Number of iterations is $3$}
\end{table}

\begin{figure}[!tbp]
  \centering
  An example in $2$ dimensions, with predictions rounded to $\{0,1\}$\\ 
  \begin{minipage}[b]{0.3\textwidth}
    \includegraphics[width=\textwidth]{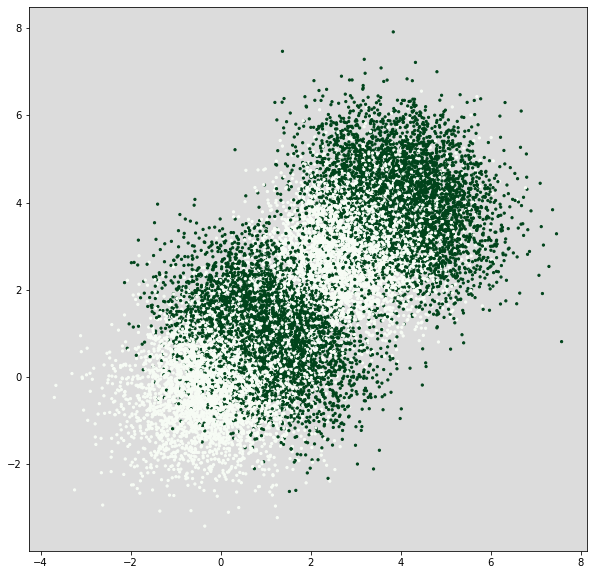}
    \caption{Ground truth}
    \label{fig:ground_truth}
  \end{minipage}
  \hfill
  \begin{minipage}[b]{0.3\textwidth}
    \includegraphics[width=\textwidth]{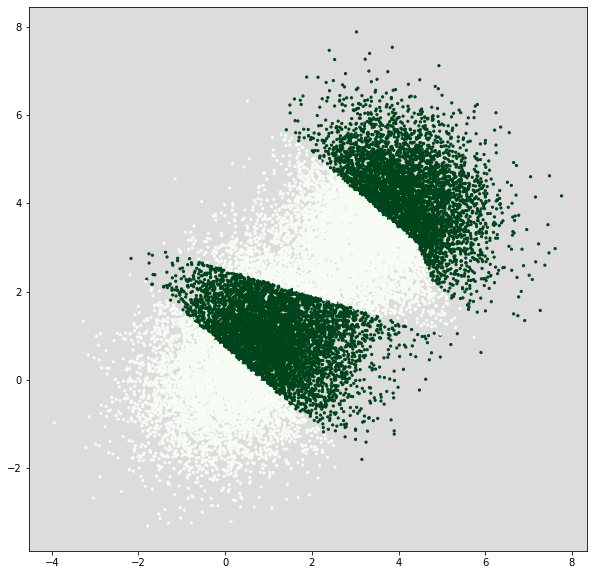}
    \caption{$\MAp$}
    \label{fig:calMA}
  \end{minipage}
  \hfill
  \begin{minipage}[b]{0.3\textwidth}
    \includegraphics[width=\textwidth]{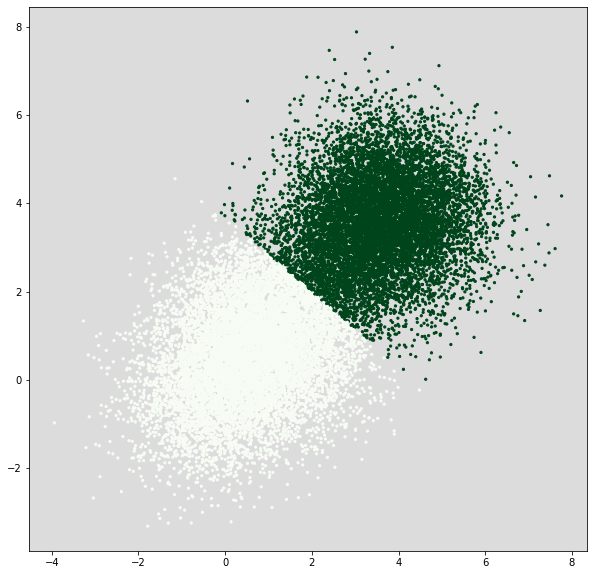}
    \caption{Linear regression}
    \label{fig:linear_regression}
  \end{minipage}
\end{figure} \subsection*{Acknowledgements}
We thank Konstantinos Satvropolous for alerting us to an error in an earlier version of this paper and suggesting a fix. \textbf{PG} and \textbf{MPK} would like to thank Mihir Singhal and Shengjia Zhao for several discussions while working on \cite{GopalanKSZ22} which inspired some of this work.  \textbf{PG} would like to thank  Adam Klivans, and Konstantinos Stavropoulos for  helpful discussions and comments on earlier versions of this paper and Raghu Meka and Varun Kanade for pointers to the literature. 
\bibliographystyle{alpha}
\bibliography{refs}
\newpage
\appendix
\section{Relaxed decision and loss OI for GLMs}
\label{app:bounded}

\begin{definition}
\label{def:approx-dec}
Given a loss $\ell$, a decision function $\hat{k}_\ell:[0,1] \to \R$ is $(\delta, D)$-approximately optimal for $\ell$ if for every $p \in [0,1]$,  
\begin{align*}
\ell(p, \hat{k}_\ell(p)) & \leq \ell(p, k_\ell(p)) + \delta.\\ 
|\hat{k}(p)| &\leq D.
\end{align*}
Let $\Lglm(\delta, D)$ denote the subset of losses in $\Lglm$ that admit an $(\delta, D)$-optimal decision function.
\end{definition}

For example, consider the logistic loss function (which is the matching loss for the sigmoid transfer function):
\[ \ell(y,t) = \log(1 + \exp((1-2y) t)).\] 
The optimal decision function is given by 
\[ k_\ell(p) = \log\lt(\frac{p}{1-p}\rt). \]
Hence  $k_\ell(0) = -\infty$, while $k_\ell(1) = \infty$. But by truncating values to the range $[\pm O(\log(1/\eps))]$, we get a decision function $\hat{k}_\ell$ which is $(\eps, O(\log(1/\eps)))$-approximately optimal.

\begin{definition}
    {\bf (Relaxed Decision OI)} 
    Let $\mL$ be a family of loss functions, and $\eps > 0$. 
    We say that predictor $\tf$ is $(\mL, \eps, \delta)$-decision-OI if for every $\ell \in \mL$ there exists $\hat{k}_\ell: [0,1] \to \R$ such that
    \begin{align*} 
    \abs{\E_\mD[\ell(\y^*, \hat{k}_\ell(\tf(\x))] - \E_{\mD(\tp)}[\ell(\ty, \hat{k}_\ell(\tf(\x))]} \leq \eps,\\
    \forall p \in [0,1], \lt| \ell(p, \hat{k}_\ell(p)) - \ell(p, k_\ell(p)) \rt| \leq \delta 
    \end{align*}
\end{definition}

For losses that admit $(\eps, D)$-optimal decision functions, we have can derive relaxed decision-OI from calibration. 

\begin{lemma}
    Every $\alpha$-calibrated predictor $\tp$ is $(\Lglm(\delta, D), D\alpha, \delta)$-decision OI.
\end{lemma}
\begin{proof}
    For every $\ell \in \Lglm(\delta, D)$, there exists a decision function $\hat{k}_\ell:[0,1] \to \R$ is $(\delta, D)$-approximately optimal for $\ell$. We repeat the proof of Part (2) of Theorem \ref{thm:dec+model} with the decision function $\hat{k}_\ell$ which gives 
    \begin{align*} 
    \abs{\E_\mD[\ell(\y^*, \hat{k}_\ell(\tf(\x))] - \E_{\mD(\tp)}[\ell(\ty, \hat{k}_\ell(\tf(\x))]} & \leq \E[|(\ty - \y^*)\partial \ell \circ \hat{k}_\ell(p(\x)|] \\
    & = \CalErr(\mW', \tp)\\
    & \leq \alpha D
    \end{align*}
    where we define the weight family $\mW' = \partial \ell \circ \hat{k}_\ell = - \hat{k}_\ell$ so that $\infnorm{\mW'} \leq D$. We then bound  $\CalErr(\mW', \tp) \leq D\alpha$ using Part (2) of Lemma \ref{lem:fully}. 
\end{proof}

We can analogously define relaxed loss OI where we consider the tests 
\begin{align} 
        \label{eq:relaxed-loss-dist}
        \hat{u}_{\ell, h}(y, \tp(x), x) = \ell(y, h(x)) - \ell(y, \hat{k}_\ell(\tp(x)) 
\end{align}
for $\ell \in \Lglm(\delta, D)$ and $h \in \Lin(\mC, B)$.  It follows that $\alpha_1$-calibration and $(\mC, \alpha_2)$ multiaccuracy suffice to fool these tests with error $D\alpha_1 + B\alpha_2$. For every $x \in \X$, we now have the inequality
\[ \E[\ell(\ty, \hat{k}_\ell(\tp(\x))|\x =x] \leq  \E[\ell(\ty, k_\ell(\tp(\x))|\x =x] + \delta \leq \E[\ell(\ty, h(\x) |\x =x] + \delta \]
we can repeat the proof of Proposition  \ref{prop:omni} to get the following claim. 

\begin{theorem}
Let $\mC$ be a bounded hypothesis class, let $B , D \geq 1$ and $\delta \geq 0$. If the  
predictor $\tp$ is $\alpha_1$-calibrated and $(\mC, \alpha_2)$-multiaccurate, then it is an
$(\Lglm(\delta, D), \Lin(\mC, B), \delta')$-omnipredictor where  
\[ \delta' = D\alpha_1 + B\alpha_2 + \delta. \]
\end{theorem}

 \section{Additional proofs}
\label{sec:app1}

\subsection{Proofs from Section \ref{sec:prelims}}
\label{app:prelims}

\begin{proof}[Proof of Lemma \ref{lem:fully}]
    Part (1). For $v \in [0,1]$, define $w^*(v) = \sign(\E[\y - v|p(\x) = v])$. For $v \not\in \Im(p)$, $w^*(v) \in [-1,1]$ can be arbitrary. Then for any $w \in W_f$, 
    \[ \E[w(p(\x))(\y - p(\x))] \leq \E[w^*(p(\x))(\y - p(\x))] = \E_{p(\x)} \lt[\lt| \E_{\y|p(\x)}[ \ \y - p(\x)] \rt|\rt] \]
    which proves the claim.

    We prove Part (2) by applying Part (1) to the family $\mW' = \mW/\infnorm{\mW}$ which consists of functions bounded in the range $[-1,1]$, and multiplying by $\infnorm{\mW}$ on either side.
\end{proof}

\subsection{Proofs from Section \ref{sec:oi-loss}}
\label{app:oi-loss}

We will use the following result of \cite{omni}. While their result applies to a broader collection of loss functions, we only state it for the class of $L_p$ losses.
\begin{theorem}\cite{omni}
\label{thm:gkrsw}
    If the predictor $\tf$ is $(\mC, \alpha)$-multicalibrated, then it is an $(L_p, \mC, 2\alpha)$-omnipredictor. 
\end{theorem}

\begin{proof}[Proof of Theorem \ref{thm:sep-mc-oi}]
    The proof uses For this proof alone, it is convenient to use the labels $\pm 1$ rather than $\zo$, and have predictors $\tf: \X \rgta [-1,1]$ where $\tf(\x) = \E[\ty|\x = x]$. 
    
      Define the distribution $(\x, \y^*) \sim \mD$ where $\mD_\X$ is uniform over $\pmo^3$ and $\y^* = \chi(\x) = \prod_{i \in [3]}x_i$ is the parity function on $3$ bits. Take $\mC = \{\sum \alpha_i x_i| \sum_i|\alpha_i| =1\}$ to be all convex combinations of $\pm x_i$. Consider the predictor $\tf(x) =0$ for all $x \in \pmo^3$  so that $\ty|\x = x$ is a uniformly random bit in $\pmo$. 
      
      We show that $\tp$ is $(\mC,0)$-multicalibrated. 
      Since $p$ is constant on the domain, multiaccuracy and multicalibration are equivalent. For reach $i \in [3]$, we have
      \[ \E_\mD[x_i(\chi(\x) - \ty)] = \E_{\mD}[x_i\chi(x)] - \E_{\mD}[\chi(\x)\y] = 0 \]
      where the first expectation is $0$ because of the orthogonality of characters, and the second because $\y$ and $\x_i$ are independent unbiased random bits. Multiaccuracy for $\mC$ now follows from linearity (Lemma \ref{lem:linear}). Now applying Theorem \ref{thm:gkrsw} implies that  $\tf$  is an $(L_p, \mC, 0)$-omnipredictor. Since calibration implies decision-OI, it also implies that $\tf$ is decision-OI with $0$ error. 
      
      \eat{Since $\tf$ is identically $0$, it is easy to verify that $k_{\ell_4}(f)$ is also identically $0$. Hence we have
      \[ \E[\ell_4(\ty, 0)]  - \E[\ell_4(\y^*, 0)] = \E[(\ty)^4] - \E[(\y^*)^4] = 1 -  1 = 0 \] 
      and decision-OI holds for $\eps = 0$. \UW{Is this clearer than saying decision-OI follows from calibration?}}
      
      Given this, it follows that hypothesis-OI and loss-OI are equivalent, and we show that neither holds. Let $c(x) = (\sum_i x_i)/3$. Then we have
      \begin{align*}
      \E[\ell_4(\ty, c(x))] &= \E[(\ty - c(x))^4]
        = ( 1 - 4\E[\ty c(x)] + 6\E[c(\x)^2] - 4\E[\ty c(\x)^3] + \E[c(\x)^4])/4\\
      \E[\ell_4(\sy, c(x))] &= \E[(\sy - c(x))^4] = (1 - 4\E[\sy c(x)] + 6\E[c(\x)^2] - 4\E[\sy c(\x)^3] + \E[c(\x)^4])/4
    \end{align*}
    Hence 
    \begin{align}
    \label{eq:diff-ell4}
        \E[\ell_4(\ty, c(\x))] - \E[\ell_4(\sy, c(\x))] &= \E[(\sy -\ty)c(\x)]  + \E[(\sy - \ty)c(x)^3]
    \end{align}
    Since $\ty$ is uniform, independent of $\x$, we have 
        $\E[\ty c(\x)] = \E[\ty c(\x)^3] =  0$.
    Since $\sy = \chi(\x)$ whereas $c(\x) = (\sum_i \x_i)/3$, by elementary Fourier analysis (see Lemma \ref{lem:tech}), we have
    \begin{align*}
        \E[\sy c(\x)] = 0, \E[\sy c(\x)^3] = 2/9
    \end{align*}
    Plugging these bounds back into Equation \eqref{eq:diff-ell4} gives
    \begin{align*}
\E[\ell_4(\ty, c(\x))] - \E[\ell_4(\sy, c(\x))] = 4/9.
    \end{align*}
This shows that hypothesis-OI and hence loss-OI do not hold.
\end{proof}

\begin{lemma}
   \label{lem:tech}
   In the setting of Theorem \ref{thm:sep-mc-oi}, we have
    \begin{align}
        \E[\sy c(\x)] &= 0,\label{eq:first}\\
        \E[\sy c(\x)^3] &= 2/9 \label{eq:third}
    \end{align}
\end{lemma}
\begin{proof}
Since $y = \chi(x) = \prod_ix_i$ and $c(x) = (\sum x_i)/3$ Equation \eqref{eq:first} follows by orthogonality of characters. For Equation \eqref{eq:third}, 
\begin{align*}
    c(\x)^3 &= \fr{27}\left(\sum_i \x_i\right)^3\\
    &= \fr{27}\left(\sum_i \x_i^3 + \sum_{i \neq j} 3\x_i^2\x_j + 6\x_1\x_2\x_3 \right)\\
    &= \fr{27}\left(4\sum_{i=1}^3\x_i + 6\prod_{i=1}^3 \x_i\right).
\end{align*}
Hence by the orthogonality of characters,
\begin{align}
    \E[\chi(\x)c(\x)^3]  = \fr{27}\E\lt[\chi(x)\lt(4\sum_i \x_i  + 6 \prod_{i=1}^3\x_i\rt)\rt] = \frac{6}{27}.
\end{align}
\end{proof}

\subsection{Proofs of Lemma \ref{lem:tech-loss}}
\label{app:glms}

Since $f(v) \geq vt -g(t)$ for all $t \in \R$ by definition, we have the Fenchel-Young inequality:
\[ \forall v \in I, t \in \R, \ \ f(v) + g(t) \geq vt. \]
Equality holds iff $g'(t) = v$. This leads us to define the Fenchel-Young divergence as:
\[ Y_{f,g}(v^*, t) = f(v^*) + g(t) - v^*t \]
which is zero iff $g'(t) = v^*$. 

\begin{lemma}
\label{lem:eq-div}
We have
\[ Y_{f,g}(v^*, t) = D_f(v^*, g'(t)) \]
\end{lemma}
\begin{proof}
    From the definitions, this is equivalent to 
    \begin{align*}
        f(v^*) + g(t) - v^*t &= f(v^*) - f(g'(t)) - (v^* -g'(t))f'(g'(t))\\
        &= f(v^*) - f(g'(t)) - (v^* -g'(t))t
    \end{align*}
    where we use $f'(g'(t)) = t$. After cancellations, this is equivalent to
    \begin{align*}
        g(t) = tg'(t) - f(g'(t))
    \end{align*}
    which is indeed true by Equation \eqref{eq:f}.
\end{proof}

\begin{proof}[Proof of Lemma \ref{lem:tech-loss}]
    For the loss function $\ell_g$, we have 
    \[ \ell_g(p^*,t) + f(p^*) = Y_{f,g}(p^*, t).\] 
    By adding $f(p^*)$ to the objective (which is independent of $h$), we can write Program \eqref{eq:breg3} as
\[\min_{h \in \Lin(\mC, B)}\E_{\x \sim \mD}[Y_{f,g}(p^*(\x),  h(\x))] \]
By Lemma \ref{lem:eq-div}, this is equivalent to Program \eqref{eq:breg2}. 
\end{proof}

\subsection{Proofs from Section \ref{sec:algo}}
\label{app:algo}

\begin{proof}[Proof of Lemma \ref{lem:ell1}]
To prove Equation \eqref{eq:first-sq}, we write
\begin{align*}
    \lt|\E[(p(\x) - p_1(\x)^2] - \E[(p(\x) - p_2(\x)^2]\rt| &= \lt|\E[(p_2(\x) - p_1(\x))( 2p(\x)  - p_1(\x) - p_2(x)]\rt|\\
    &\leq 2\E[|p_2(\x)  - p_1(\x)|]\\
    &\leq 2\delta.
\end{align*}

The bound on multiaccuracy follows by observing that for  any $c :\X \rgta [-1,1]$,
\begin{align*}
    \E[c(\x)(\y - p_1(\x))] - \E[c(x)(\y - p_2(\x))] &= \E[c(\x)(p_2(\x) - p_1(\x))]\\
     & \leq \E[|p_2(\x) - p_1(\x)|] \leq \delta. 
\end{align*}
\end{proof}

\begin{proof}[Proof of Lemma \ref{lem:eff-cal}]
    We take a set of $m = O(\log^2(1/\delta)/(\delta\mu^3))$ samples $(x, y)$ and compute $\pone(x)$ for each. For each $j \leq 1/\delta$, let $S_j$ denote the set of samples where $\pone(x) = j\delta$ and $m_j = |S_j|$. Define the values
    \begin{align}
    \label{def:yj}
        \bar{y}_j &= \fr{m_j} \sum_i y_i \\
        \err_j &= |\bar{y}_j - j\delta|\notag \\
        \estECE &= \sum_{j=0}^{1/\delta} \frac{m_j}{m}\err_j\notag
    \end{align}
    The algorithm returns the value $\estECE$.
    
    We ignore any small values of $j$ such that $\Pr[p^\delta(\x) = j\delta] \leq \mu\delta/4$, since except with probability $0.1$, such values only contribute $\mu/4$ to $|\estECE - \ECE(\pone)|$. Call the other values of $j$ large. For every large $j$,  we have by Chernoff bounds, we have
    \begin{align*}
        \Pr[m_j \leq C(\log(1/\delta)/\mu^2)]  \leq \frac{\delta}{30}
    \end{align*}
    Assuming this event holds, we have
    \begin{align*}
         \Pr\lt[\lt|\Pr[p(\x) \in I_j] - \frac{m_j}{m} \rt| \geq \frac{\mu}{4}\rt] \leq \frac{\delta}{30}\\
         \Pr\lt[|\bar{y_j} - \E[\y|\pone(\x) = j\delta] | \geq \frac{\mu}{4}\rt] \leq \frac{\delta}{30}.
    \end{align*}
    We take a union bound over all $1/\delta$ large values. 
    Except with error probability $0.2$, none of the bad events considered above occur, and we have $|\estECE - \ECE(\pone)| \leq \mu|$. We can reduce the failure probability by repeating the estimator and taking the median. For simplicity, we ignore the failure probability.
    
    To define the predictor $\pth$, we repeat the analysis above with $\mu = \delta$. We define $\pth(x) = \bar{y}_j$ for all $x \in I_j$.  We show that it is close to $\ptwo$ in $\ell_1$. The contribution of small values of $j$ to $\E[|\ptwo(\x) - \pth(\x))|]$ is no more than $\mu/4$. For large buckets, we have 
    \[ |\bar{y} - \ptwo(x)| \leq  \lt|\bar{y}_j - \E[\y|p^\delta(\x) = j\delta]\rt| \leq \delta/2 + \mu/4. \]
    Thus overall, the distance is bounded by $(\delta/2 + \mu/4) \leq \delta$ by our choice of $\mu$. 
    
    Lastly, we bound the calibration error, using the fact that $\ptwo$ is perfectly calibrated, and $\pth$ is close to it $\ptwo$. Note that both $\ptwo$ and $\pth$ are constant on all $x \in p^{-1}(I_j)$. Hence
    \begin{align*} 
    \ECE(\pth) &= \E_{\mb{I_j}}\lt| \E_{p(\x) \in \mb{I_j}}[\y - \pth(\x)] \rt| \\
    & \leq   \E_{\mb{I_j}}\lt| \E_{p(\x) \in \mb{I_j}}[\y - \ptwo(\x)] \rt| + \E_{\mb{I_j}}\lt| \E_{p(\x) \in \mb{I_j}}[\ptwo(\x) - \pth(\x)] \rt| \\
    &= \E[|\ptwo(\x) - \pth(\x)|] \\
    & \leq \delta.
    \end{align*}
\end{proof}

\begin{proof}[Proof of Corollary \ref{cor:err-red}]
By Lemma \ref{lem:err-red}, we know that
\begin{align*}
       l_2(p^*, \pone)^2  -  l_2(p^*, \ptwo)^2  \geq \ECE(\pone)^2.
\end{align*}
By Lemma \ref{lem:eff-cal}, $l_1(\ptwo, \pth) \leq \delta$. Hence Lemma \ref{lem:ell1} implies that 
\begin{align*}
    \lt|l_2(p^*, \pth)^2 - l_2(p^*, \ptwo)^2\rt| \leq 2\delta.
 \end{align*}
 Similarly, since $l_\infty(p, \pone) \leq \delta$, 
 \begin{align*}
    \lt|l_2(p^*, p)^2 - l_2(p^*, \pone)^2\rt| \leq 2\delta.
 \end{align*}
 The claim follows by combining these three equations.
\end{proof}

\begin{proof}[Proof of Lemma \ref{lem:hkrr}]
Since $\WL(p^* - p_T) = \perp$, by the definition of the weak agnostic learner, for every $c \in \mC$,
    \[\E[c(\x)(p^*(\x) - p_T(\x))] = \E[c(\x)(\y^* - p_T(\x))] \leq \sigma. \]
    Since $\mC$ is closed under negation, the bound also holds in absolute value, hence $p_T \in \MA(\rho) \subseteq \MA(\alpha)$ since $\alpha \geq \rho$. 

Assume that $T \geq 1$ and let $t \in \{0, \ldots, T-1\}$. We consider the change in the expected squared error of $p_t$ for every iteration. Note the \begin{align*}
    l_2(p^*, p_t)^2 = \E[(p^*(\x) - p_t(\x))^2] \leq \E[(p^*(\x) - h_{t+1}(\x))^2]
\end{align*}
since $p_t = \Pi(h_t)$ and projection can only reduce squared error. 
\begin{align*}
    l_2(p^*, p_t)^2 - l_2(p^*, p_{t+1})^2
    &= \E[(p^*(\x) - p_t(\x))^2] - \E[(p^*(\x) - p_{t+1}(\x))^2]\\
    & \geq \E[(p^*(\x) - p_t(\x))^2] - \E[(p^*(\x) - h_{t+1}(\x))^2] \\
    &= \E[(h_{t+1}(\x) - p_t(\x))(2p^*(\x) - p_t(\x) -h_{t+1}(\x)]\\
    &= \E[\sigma c'_t(\x)(2p^*(\x) - 2p_t(\x) - \sigma c'_t(\x))\\
    &= 2\sigma \E[ c'_t(\x)(p^*(\x) - p_t(\x))] - \sigma^2\E[c'_t(\x)^2]\\
    &\geq 2\sigma^2 - \sigma^2 = \sigma^2.
\end{align*}
We sum this over $t  \in \{0,\ldots, T-1\}$ to get 
 \[ \E[ (p^*(\x) - p_0(\x))^2]  - \E[(p^*(\x) - p_T(\x))^2] \geq T\sigma^2.\qedhere\]
 \end{proof}

\subsection{Complexity analysis for $\MC$}
\label{app:omni}

The algorithm from \cite{omni} sets a parameter $\delta \approx \alpha\sigma^2$. It maintains a set of $m = O(1/\delta)$ states, where state $j$ represents a prediction of $j\delta$. In each epoch (where multiple Split operations and a single Merge operation occur), the squared error drops by $\delta$ at the cost of $O(m)$ calls to $\WL$. This implies a total of $O(m/\delta) = O(1/\delta^2)$ calls to $\WL$ till termination. 

One can view the weak learning problem as having a $(\rho, \sigma)$-weak learner for arbitrary marginal distributions on $\X$. Alternately, we can stick to the same marginal distribution $\D_\X$, but we need to make a stronger assumption on the weak learner, which is for every $\rho \geq 0$, there exists $\sigma(\rho)$ such that if there exists  $c \in \mC$ such that $\E_\mD[c(\x)f(\x)] \geq \rho$, then $\WL(f, \rho) = c' \in \mC$ such that $\E_{\mD}[c'(\x)f(\x))] \geq \sigma(\rho)$.
The stronger form of the weak learner is required since the algorithm might present the weak learner with distributions on labels where the correlation with $c$ is rather small; roughly $\Omega(\alpha^3 \sigma^3)$, and require it to find a non-trivially correlated hypothesis.

\end{document}